\documentclass{article}

    \usepackage[final]{neurips_2022}

\usepackage[utf8]{inputenc} % allow utf-8 input
\usepackage[T1]{fontenc}    % use 8-bit T1 fonts
\usepackage{hyperref}       % hyperlinks
\usepackage{url}            % simple URL typesetting
\usepackage{booktabs}       % professional-quality tables
\usepackage{amsfonts}       % blackboard math symbols
\usepackage{nicefrac}       % compact symbols for 1/2, etc.
\usepackage{microtype}      % microtypography
\usepackage{xcolor}         % colors
\usepackage{amsmath}
\usepackage{amsthm}
\usepackage{mathrsfs}
\usepackage{float}
\usepackage{graphicx}
\usepackage{caption,subcaption}
\usepackage{bbm}
\usepackage{multirow}
\usepackage[ruled,linesnumbered]{algorithm2e}

\title{Generalization Bounds for Gradient Methods via Discrete and Continuous Prior}

\author{%
Xuanyuan Luo \\
IIIS, Tsinghua University \\
\texttt{xuanyuanluo@google.com} \\
\And
Luo Bei \\
Renmin University of China\\
\texttt{rabbit\_lb@ruc.edu.cn} \\
\And
Jian Li\\
IIIS, Tsinghua University\\
\texttt{lijian83@mail.tsinghua.edu} \\
}
\hypersetup{colorlinks, 
            linkcolor=blue,
            citecolor=blue,
            urlcolor=magenta,
            linktocpage,
            plainpages=false}
\begin{document}

\newcommand{\topic}[1]{\noindent \textbf{#1}}
\newcommand{\eat}[1]{}

\newcommand{\diam}{\mathrm{diam}}
\newcommand{\diag}{\mathrm{diag}}
\newcommand{\DTD}[2]{\Delta^{*}\left(#1, #2\right)}
\newcommand{\err}{\mathrm{err}}
\newcommand{\Ex}[2]{\operatorname*{\mathbb{E}}_{#1}\left[#2\right]}
\newcommand{\Var}{\mathrm{Var}}
\newcommand{\E}{\mathop{\mathbb{E}}}
\newcommand{\alg}{A}
\newcommand{\gen}{\textrm{gen}}
\newcommand{\sgn}{\mathrm{sgn}}
\newcommand{\gld}{\mathrm{GLD}}
\newcommand{\erf}{\mathrm{erf}}
\newcommand{\erfc}{\textrm{erfc}}
\newcommand{\generr}{\err_{\gen}}
\newcommand{\vc}{\mathrm{VC}_{\mathrm{dim}}}
\newcommand{\Egen}{\err_{\gen}}
\newcommand{\gpop}{\mathbf{g}_{\mathrm{p}}}
\newcommand{\gemp}{\mathbf{g}_{\mathrm{e}}}
\newcommand{\KL}[2]{\mathrm{KL}\left(#1\,\big|\big|\,#2\right)}
\newcommand{\N}{\mathcal{N}}
\newcommand{\pack}{\mathcal{P}}
\newcommand{\noise}{\textrm{noise}}
\newcommand{\dif}{\mathrm{dif}}
\newcommand{\norm}[1]{\left\|#1\right\|}
\newcommand{\opnorm}[1]{\left\|#1\right\|_{\textrm{op}}}
\newcommand{\rkhsnorm}[1]{\left\|#1\right\|_{\mathcal{K}}}
\newcommand{\R}{\mathbb{R}}
\newcommand{\risk}{\mathcal{R}}
\newcommand{\Rad}{\mathfrak{R}}
\newcommand{\TD}[2]{\Delta\left(#1, #2\right)}
\newcommand{\tr}{\mathrm{tr}}
\newcommand{\uni}{\mathrm{uniform}}
\newcommand{\bd}{\mathrm{bound}}
\newcommand{\TV}{\mathrm{TV}}
\newcommand{\cond}{\mathcal{E}}
\newcommand{\loss}{\mathcal{L}}
\newcommand{\acc}{\mathcal{L}^{01}}
\newcommand{\Z}{\mathcal{Z}}
\newcommand{\W}{\mathcal{W}}
\newcommand{\tra}{\mathrm{train}}
\newcommand{\set}{\mathrm{set}}
\newcommand{\te}{\mathrm{test}}
\newcommand{\round}{\mathrm{round}}
\newcommand{\floor}{\mathrm{floor}}
\newcommand{\sn}{\mathrm{sign}}
\newcommand{\eps}{\varepsilon}
\newcommand{\p}{\partial}
\newcommand{\rmd}{\mathrm{d}}
\newcommand{\Ent}{\mathrm{Ent}}
\newcommand{\fcnngap}{f_{\mathrm{cnngap}}}
\newcommand{\bfP}{\mathbf{P}}
\newcommand{\g}{\nabla} %means gradient
\newcommand{\gdiff}{\mathbf{g}} %means gradient

\newcommand{\0}{\mathbf{0}}
\newcommand{\D}{\mathcal{D}}
\newcommand{\oS}{\overline{S}}
\newcommand{\e}{\varepsilon}
\newcommand{\id}{\mathbbm{1}}
\newcommand{\pp}[2]{\frac{\p #1}{\p #2}}
\newcommand{\dd}[2]{\frac{\rmd #1}{\rmd #2}}
\newcommand{\LR}[1]{\left(#1\right)}
\newcommand{\abs}[1]{\left\|#1\right\|}
\newcommand{\mzero}[1]{(#1)^+}
\newcommand{\ag}[2]{\langle #1, #2 \rangle}
\newcommand{\train}{\mathrm{train}}
\newcommand{\real}{\mathrm{test}}
\newcommand{\net}{\mathrm{net}_W}
\newcommand{\softmax}{\mathrm{softmax}}
\newcommand{\relu}{\mathrm{relu}}
\newcommand{\thmid}[1]{{\bf \hspace{-0.1cm}{#1}}. }
\newtheorem{definition}{Definition}[section]
\newtheorem{theorem}[definition]{Theorem}
\newtheorem{lemma}[definition]{Lemma}
\newtheorem{corollary}[definition]{Corollary}
\newtheorem{assumption}[definition]{Assumption}
\newtheorem{proposition}[definition]{Proposition}
\newtheorem{remark}[definition]{Remark}
\newtheorem{example}[definition]{Example}

\newtheorem*{ctheorem}{Theorem}
\newtheorem*{clemma}{Lemma}
\newtheorem*{ccorollary}{Corollary}
\maketitle
\begin{abstract}
Proving algorithm-dependent generalization error bounds for gradient-type optimization methods has attracted significant attention recently in learning theory. 
However, most existing trajectory-based analyses require
either restrictive assumptions on the learning rate (e.g., fast decreasing learning rate), or  continuous injected noise (such as the Gaussian noise in Langevin dynamics).
In this paper, we introduce a new discrete data-dependent prior to the PAC-Bayesian framework,
and prove high probability generalization bounds of order $O(\frac{1}{n}\cdot \sum_{t=1}^T(\gamma_t/\varepsilon_t)^2\left\|{\gdiff_t}\right\|^2)$  for floored GD and SGD (i.e. finite precision versions of GD and SGD with precision level $\varepsilon_t$) where, where $n$ is the number of training samples, $\gamma_t$ is the learning rate at step $t$, 
$\gdiff_t$ is roughly the difference between the average gradient over all samples and that over only prior samples. $\left\|{\gdiff_t}\right\|$ is upper bounded by (typically much smaller) than 
the gradient norm $\left\|{\g f(W_t)}\right\|$. 
We remark that our bounds hold for nonconvex and nonsmooth loss functions. 
Moreover, our theoretical results provide numerically favorable upper bounds of testing errors
($0.026$ on MNIST and $0.198$ on CIFAR10). 
%Using similar technique, we can also obtain new generalization bounds for certain variant of SGD.
Furthermore, we study the generalization bounds for gradient Langevin Dynamics (GLD). 
Using the same framework with a carefully constructed continuous prior,
we show a new high probability generalization bound of order $O(\frac{1}{n} + \frac{L^2}{n^2}\sum_{t=1}^T(\gamma_t/\sigma_t)^2)$ for GLD.
The new $1/n^2$ rate is obtained using the concentration of the difference between the gradient of training samples and that of the prior.
\end{abstract}

\section{Introduction}
Bounding generalization error of learning algorithms is one of the most 
important problems in machine learning theory. Formally, for a supervised learning problem, 
the generalization error is defined as the testing error (or population error) minus the training error (or empirical error).
In particular, we denote $\risk(w,(x,y)) := \id[h_w(x)\neq y]$ as the error of a single data point $(x,y)$, where $h_w(x)$ is the output of a model with parameter $w\in \R^d$. Suppose
$S$ is the set of training data, each i.i.d. sampled from the population distribution $\D$,
and we use $\risk(w,S):=\frac{1}{|S|}\sum_{z\in S}\risk(w,z)$ and $\risk(w, \D):=\E_{z\sim D}[\risk(w,z)]$ to denote the training error and the testing error, respectively. 
The generalization error of $w$ is formally defined as 
$\generr(w)=\risk(w,\D)-\risk(w,S)$.

Proving tighter generalization bounds for general nonconvex learning and  particularly deep learning has attracted significant attention recently. 
While the classical learning theory (uniform convergence theory) which bounds the generalization error by various complexity measures (e.g., the VC-dimension and Rademacher complexity) of the hypothesis class has been successful in several classical convex learning models, however, they become vacuous and hence fail to explain the success of modern nonconvex over-parametrized neural networks (i.e., the number of parameters significantly exceeds the number of training data) (see e.g., \citet{ZhangBHRV17,nagarajan2019uniform}). Recently, learning theorists have tried to understand and explain generalization of deep learning from several other perspectives, such as margin theory ~\citep{bartlett2017spectrally, wei2019regularization}, algorithmic stability~\citep{hardt2016train,mou2018generalization, li2019generalization, bousquet2020sharper}, PAC-bayeisan~\citep{london2017pac, bartlett2017spectrally,neyshabur2018pac,ZhouVAAO19, yang2019fast}, neural tangent kernel~\citep{jacot2018neural,du2019gradient,arora2019fine,cao2019generalization}, information theory~\citep{pensia2018generalization,negrea2019information}, model compression~\citep{arora2018stronger,ZhouVAAO19}, differential privacy~\citep{oneto2017differential,wu2021generalization} and so on. 

In this paper, we aim to obtain tighter generalization error bounds that depend on both the training data and the optimization algorithms (a.k.a. gradient-type methods) for general nonconvex learning problems. In particular, we prove algorithm-dependent generalization bounds for several gradient-based optimization algorithms such as certain variants of gradient descent (GD), stochastic gradient descent (SGD) and stochastic gradient Langevin dynamics (SGLD).
Our proofs are based on the classic Catoni's PAC-Bayesian framework~\citep{catoni2007pac} and also have a flavor of algorithmic stability~\citep{bousquet2002stability}. 
Several prior works have obtained generalization bounds for SGD and SGLD by analyzing trajectory through either the PAC-Bayesian or the algorithmic stability framework  (or closely related information theoretic arguments). However, most existing results
based on analyzing the optimization trajectories require
either restrictive assumptions on the learning rates, or continuous noise (such as the Gaussian noise in Langevin dynamics) in order to bound the stability or the KL-divergence. 
In this paper, we resolve the above restrictions by combining the PAC-Bayesian framework with a few simple (yet effective) ideas, so that we can obtain new high probability and non-vacuous generalization bounds for several gradient-based optimization methods with either discrete or continuous noises (in particular certain variants of GD and SGD, either being deterministic or with discrete noise, which cannot be handled by existing techniques). 

\subsection{Prior work}
We first briefly mention some recent work on bounding the generalization error 
of gradient-based methods.
\citet{hardt2016train} first studied the uniform stability (hence
the generalization) of stochastic gradient descent (SGD) for both convex and non-convex functions. Their results for non-convex functions requires that the learning rate $\eta_t$ scales with $1/t$. Their work motivates a long line of subsequent work on generalization error bounds of gradient-based
optimization methods:~\citet{kuzborskij2018data,london2016generalization,chaudhari2019entropy,raginsky2017non, mou2018generalization,chen2018stability,li2019generalization,negrea2019information,wang2021analyzing}. 

Recently, {\citet{simsekli2020hausdorff,hodgkinson2022generalization} obtained 
generalization bound of SGD through the perspective of heavy-tailed behaviors and using the notion of Hausdorff dimension $\mathrm{d_H}$ which depends on both the algorithm and data. 
}

\vspace{-0.25cm}
\paragraph{PAC-Bayesian bounds.} The PAC-Bayesian framework~\citep{mcallester1999some} is a powerful method for proving high probability generalization bound~\citep{bartlett2017spectrally,ZhouVAAO19,mou2018generalization}. Roughly speaking, it bounds the generalization error by the KL divergence $\KL{Q}{P}$, where $Q$ is the distribution of the learned output and $P$ is a prior distribution which is typically independent of dataset $S$. In this framework, bounding $\KL{Q}{P}$ is the most crucial part for obtaining tighter PAC-Bayesian bounds. In order to bound the KL divergence, both the prior $P$ and posterior $Q$ are typically chosen to be continuous distributions (mostly Gaussians so that KL can be computed in closed form).
Hence, most prior work either considered gradient methods with continuous noise (such as Gradient Langevin Dynamics) (e.g.,~\citep{mou2018generalization,li2019generalization,negrea2019information}), or injected a Gaussian noise to the final parameter at the end
(e.g.,~\citep{neyshabur2018pac,ZhouVAAO19}) (so $Q$ is a Gaussian distribution).
We also note that designing effective prior $P$ can be also very important.
For example, \citet{lever2013tighter} proposed to use the population distribution to compute the prior. In fact, the prior can even partially depend on the training data~\citep{parrado2012pac,negrea2019information},
and our Theorem~\ref{thm:data-pac} is partially inspired by this idea.
\subsection{Our contributions}
First, we provide high probability generalization bounds for \emph{discrete} gradient methods.
In particular, we study the generalization of Floored Gradient Descent (FGD), which is a variant of GD, and Floored Stochastic Gradient Descent (FSGD), a variant of SGD. We obtain our bound by an interesting construction of discrete priors. Secondly, we consider well studied gradient methods with continuous noise, (stochastic) gradient Langvin dynamics (GLD and SGLD).
We show sharper generalization bounds by carefully bounding the concentration of the sample gradients. Now, we summarize our results.
\vspace{-0.2cm}
\paragraph{FGD and FSGD.} 
We first study an interesting variant of GD, called Floored GD (FGD) (Algorithm~\ref{alg:fgd}). 
The update rule of FGD is defined as follows:
\begin{align*}\label{eq:fgd}
W_{t} \gets W_{t-1} - \gamma_t \g f(W_{t-1}, S_J) - \eps_t \floor\left(\gamma_t \gdiff_t/\eps_t\right), \tag{FGD} 
\end{align*}
where $S_J$ is the subset of training dataset $S$ with size $m$ indexed by subset $J\subset [n]$ ($J$ is chosen before training), $\g f(W_{t-1}, Z):=\frac{1}{|Z|}\sum_{z\in Z} \g f(W_{t-1}, z)$ is the average gradient over the dataset $Z$, $\gamma_t$ is the learning rate, $\eps_t$ is the precision level, and $\gdiff_t:=\g f(W_{t-1},S) - \g f(W_{t-1},S_J)$ is the gradient difference. The flooring operation is defined by $\floor(x):=\sn(x)\lfloor|x|\rfloor$ for any real number $x$. 
FGD can viewed as GD with given precision limit $\eps_t$. 
We can see if we ignore the floor operation or let $\eps_t$ approaches 0, FGD reduces to GD (see also Appendix~\ref{app:fgdintro}).

We also study a finite precision variant of SGD, called Floored SGD (FSGD) 
(see Section~\ref{sec:fgd} for its formal definition).
Empirically, the optimization and generalization
capabilities of FGD and FSGD are
very close to those of GD and SGD
(see Figure~\ref{fig:fgd-gd-mnist} and~\ref{fig:fsgd-sgd-cifar} in Appendix~\ref{app:exp-detail}).

%We can generalize it to a vector by applying the flooring operation to each dimension.

By constructing a discrete data-dependent prior and incorporate it into Catoni's PAC-Bayesian framework, we prove that the following bound (Theorem~\ref{thm:fgd}) holds for FGD with high probability:
\begin{align*}
\risk(W_T,\D) &\leq c_0\risk(W_T,S_{[n]\backslash J}) + O\left(\frac{1}{n-m}
+\frac{\ln(dT)}{n-m}\sum_{t=1}^T\frac{\gamma_t^2}{\eps_t^2}\norm{\gdiff_t}^2\right),
\end{align*}
where $d$ is the dimension of parameter space and $c_0$ can be chosen to be a small constant. The bound for FSGD is very similar (see Theorem~\ref{thm:fsgd}).
Now we make a few remarks about our results.
\vspace{-0.2cm}
\begin{enumerate}
    \item Our result holds for nonconvex and nonsmooth learning problems (replacing the gradients with subgradients for nonsmooth cases). Moreover, there is no additional requirement on the learning rate $\gamma_t$.
    \item The gradient difference $\gdiff_t$ is typical much smaller than the worst case gradient norm. It usually decreases when $m=|J|$ grows (see Figure~\ref{fig:fgd-mnist-exp-c} in Section~\ref{sec:exp}). 
    \item 
    We obtain non-vacuous generalization 
    bounds on commonly used datasets. 
    Specifically, our theoretical test error upper bounds on MNIST and CIFAR10 are $\mathbf{0.026}$ and $\mathbf{0.198}$, respectively (see Section~\ref{sec:exp}). Both of them are tighter than the best-known MNIST bound ($11\%$) and CIFAR10 bound ($23\%$) reported in \citet{dziugaite2021role}. 
    See Table~\ref{tab:compare} in Appendix~\ref{app:compare} for more 
    comparisons.
    \item 
    In order to bound the KL between $P$ and the deterministic process of FGD, we construct the prior $P$ from a
    discrete random processes..
    We hope it may inspire future research on handling deterministic optimization algorithms or discrete noise.
\end{enumerate}

\noindent
{\bf Why study FGD/FSGD?} We would like to remark that we study FGD/FSGD, not because FGD/FSGD have better performances than GD/SGD or other advantages.
Indeed, their performances are almost the same as those of GD/SGD (see Appendix~\ref{app:exp-detail}). We use them as important stepping stones to study generalization bounds for GD and SGD. 
Note that most existing trajectory-based generalization bounds require either fast decreasing learning rate, or continuous injected noise, such as the Gaussian noise in Langevin dynamics, for general non-convex loss functions. Handling deterministic algorithms (such as GD) or discrete noises (such as SGD) is challenging and beyond the reach of existing techniques. 
%Our new technique significantly extends the applicability of PAC-Bayesian.
In fact, understanding such discrete noises and their effects on generalization has been an important research topic (see e.g., \citet{li2019generalization,zhu2019anisotropic,ziyin2021strength}). 
In particular, \citet{zhu2019anisotropic} show that it is  insufficient to approximate SGD's discrete noise by isotropic Gaussian noise. Moreover, proving nontrivial generalization bounds for SGD-like algorithms with discrete noise
has also been proposed as an open research direction in \citet{li2019generalization}. 
 % end of color

\paragraph{GLD and SGLD.} We provide a new generalization bound for Gradient Langevin Dynamics 
(GLD). The update rule of GLD is defined as follows.
\begin{align*}
W_t \gets W_{t-1} + \gamma_t \g f(W_{t-1},S) + \sigma_t \N(0,I_d). \tag{GLD}
\end{align*}
In this paper, we show that the following generalization bound (Theorem~\ref{thm:gld-bound}) holds with high probability over the randomness of $S\sim \D^n$ and random subset $J\subset [n]$
($|J|=m$):
\begin{align*}
\risk(W_T,\D) \leq c_0\risk(W_T,S_{[n]\backslash J}) + O\left( \frac{1}{n-m} + \frac{1 }{(n-m)m}\E\left[\sum_{t=1}^T\frac{\gamma_t^2}{\sigma_t^2}L(W_{t-1})^2\right]\right),
\end{align*}
where $L(W_{t-1}):=\max_{z \in S}\norm{f(W_{t-1},z)}$ is the longest gradient norm of any 
training sample in $S$ at step $t$ and $m$ is the size of $J$. 
Since $W_T$ is independent of the index set $J$, the first term $\risk(W_T, S_{[n]\backslash J})$ is upper bounded by $\risk(W_T,S) + O(\frac{1}{\sqrt{n-m}})$ with high probability,
using standard Hoeffding's inequality. 
By setting $m = n/2$, our generalization bound has an $O(\frac{1}{\sqrt{n}} + \frac{1}{n} + \frac{T}{n^2})$ rate. 
The new $1/n^2$ rate is obtained using the concentration of the difference between the gradient of training samples and that of the prior (See Lemma~\ref{lem:gld-grad-con}).

We also prove a high probability generalization bound for 
Stochastic Gradient Langevin Dynamics (SGLD) (see Theorem~\ref{thm:sgld-bound}): \begin{align*}
\risk(W_T,\D) \leq c_0\risk(W_T,S_{[n]\backslash J}) +  O\left(\frac{1}{n-m} +
\frac{1}{n-m}\left(\frac{1}{b} + \frac{1}{m}\right) \E\left[\sum_{t=1}^T\frac{\gamma_t^2}{\sigma_t^2}L(W_{t-1})^2\right]\right).
\end{align*}
We compare our bounds with other GLD/SGLD bounds obtained in~\citep{mou2018generalization,negrea2019information,li2019generalization} and the details can be found in Appendix~\ref{app:compare}.

\paragraph{CLD.}
Using the PAC-Bayesian framework, we obtain a new generalization bound for Continuous Langevin Dynamics (CLD), defined by
the stochastic differential equation
$\rmd W_t = -\g F(W_t,S)~\rmd t+ \sqrt{2\beta^{-1}}~\rmd B_t$.
The main term of the generalization bound scales as $O(1/n^2)$ (by choosing $m=n/2)$ and does not grow to infinity as the training time $T$ increases. See Theorem~\ref{thm:cld-bound} for the details.

\section{Other Related Work}
\vspace{-0.2cm}
\paragraph{Stochastic Langevin Dynamics}
Stochastic Langevin dynamics is a popular sampling and optimization method in machine learning~\citep{welling2011bayesian}.
\citet{zhang2017hitting,chen2020stationary} show a polynomial hitting time (hitting a stationary point) of SGLD in general non-convex setting.
\citet{raginsky2017non} study the generalization and excess risk of SGLD in nonconvex settings and
their bound depends inversely polynomially on a certain spectral gap parameter, which may be exponential small in the dimension. Continuous Langevin dynamics (SDE) with various noise structure has also been used extensively as approximations of SGD
in literature (see e.g., \citep{li2017stochastic,li2021validity}). 
{
However, in terms of 
generalization, isotropic Gaussian noise is not a good approximation of the discrete noise in SGD (\cite{zhu2019anisotropic}).
}
\vspace{-0.25cm}
\paragraph{Nonvacuous PAC-Bayesian Generalization Bounds.} 
\citet{dziugaite2017computing} first present a non-vacuous PAC-Bayesian generalization bound on MNIST (0.161 for a 1-layer MLP, see column T-600 of Table 1 in their paper). They use a very different training algorithm that explicitly optimizes the PAC-Bayesian bound and the output distribution is a multivariate normal distribution. 
To computing the closed form of KL, they choose a zero-mean Gaussian distribution as the prior distribution.
\citet{ZhouVAAO19} obtain the first non-vacuous generalization bound for ImageNet via a different method.
Their method does not require any continuous noise injected but assumes that the network can be significantly compressed (so that the prior distribution is supported over the set of discrete parameters with finite precision).
To our best knowledge, it is the only work that utilizes a discrete prior for proving generalization bounds of deep neural networks. Our result for FGD/FSGD has a similar flavor in a high level, that is the optimization method has a finite precision. However, our results do not need any assumption on compressibility of the model and can be applied to nonconvex learning problems other than neural networks.
\vspace{-0.25cm}
\paragraph{Generalization bounds via Information theory.} 
\citet{raginsky2017non} first show that the expected generalization error $\E_{S\sim \D^n}[\risk(W,\D) - \risk(W,S)]$ is bounded by $\sqrt{2I(S;W)/n}$, where $I(S;W):=\KL{P(S,W)}{P(S)\otimes P(W)}$ is the mutual information
between the data set $S$ and the parameter $W$. 
This work motivates several subsequent studies~\citep{pensia2018generalization, negrea2019information, bu2020tightening, wang2021analyzing}. The main goal in this line of work is to obtain a tight bound on the mutual information $I(S;W)$. This is again reduced to bounding the KL divergence and thus typically requires continuous injected noise (e.g., \cite{wang2021analyzing,negrea2019information}).

\section{Preliminaries}\label{sec:prelim}
\vspace{-0.2cm}
\paragraph{Notations.} We assume that the training dataset $S=(z_1,..,z_n)$ is sampled from $\D^n$, where $\D$ is the population distribution over the data domain $\Omega$. The model parameter $w$ is in $\R^d$. 
The risk function $\risk:\R^d\times \Omega \to [0, 1]$ measures the error of a model on a datapoint. The loss function $f:\R^d\times \Omega \to \R$ is a proxy of the risk.
The optimization algorithm minimizes the loss function and we assume we can compute the gradient of the loss function. We note that the loss function may be different from the risk function (e.g., 0/1 risk vs the cross-entropy loss). The empirical risk is $\risk(w, S)=\frac{1}{|S|}\sum_{z\in S}\risk(w,z)$ and population risk is $\risk(w,\D)=\E_{z\sim \D}[\risk(w,z)]$. 
Similarly, we can define the empirical loss $f(w, S)$ and population loss $f(w, \D)$. 
For any $J=(j_1,..,j_m)$, we use $S_J$ to denote the sequence $(S_{j_1},...,S_{j_m})$. The 
subsequence $(A_i,A_{i+1},...,A_j)$ is denoted by $A_i^j$. We use $(A_1^n,B_1^m)$ to denote the merged sequence $(A_1,A_2,...,A_n, B_1,...,B_m)$. When the elements in sequence $J$ are distinct, we also use $J$ to represent the set consisting of all of its elements. We may also slightly abuse the notation of a random variable to denote its distribution. For example, $\E_{x\sim X}[f(x)]$ is a shorthand for $\E_{x\sim P_X}[f(x)]$, and $\KL{X}{Y}$ means $\KL{P_X}{P_Y}$. For a random variable $W$, we define $\risk(W,S)=\E_{w\sim W}[\risk(w,S)]$ and $\risk(W,\D) = \E_{w\sim W}[\risk(w,\D)]$. The set $\{1,2,...,n\}$ is denoted by $[n]$.
\vspace{-0.3cm}
\paragraph{KL-divergence.}
Let $P$ and $Q$ be two probability distributions. 
The Kullback–Leibler divergence $\KL{P}{Q}$
is defined only when $P$ is absolute continuous with respect to $Q$ (i.e., for any $x$, $Q(x) = 0$ implies $P(x) = 0$). In particular, if $P$ and $Q$ are discrete distributions, then $\KL{P}{Q}=\sum_{x}P(x)\ln\frac{P(x)}{Q(x)}$. Otherwise, if $P$ and $Q$ are continuous distributions, it is defined as 
$\int P(x)\ln\frac{P(x)}{Q(x)}~\rmd x$. The following Lemma~\ref{lem:chain-kl} is frequently used in this paper and is a well known property of KL divergence (see \citet[Theorem 2.5.3]{cover1999elements}, \cite{li2019generalization}, \cite{negrea2019information}).
\begin{lemma}[Chain Rule of KL]\label{lem:chain-kl}
We are given two random sequences $W=(W_0,...,W_T)$ and $W'=(W'_0,...,W'_T)$.
%if for any set $A$ we have $\Pr[W'\in A] = 0$ implies $\Pr[W\in A] = 0$. 
Then, the following equation holds (given all KLs are well defined):
\begin{align*}
\KL{W}{W'} = \KL{W_0}{W'_0}+\sum_{t=1}^T\E_{w\sim W_0^{t-1}}\left[\KL{W_t|W_0^{t-1}=w}{W'_t|{W'}_0^{t-1}=w}\right].
\end{align*}
Here $W_t|W_0^{t-1}=w$ denotes the distribution of $W_t$ conditioning on 
$W_0^{t-1}=(W_0,\ldots, W_{t-1})=w$.
\end{lemma}
\vspace{-0.3cm}
\paragraph{PAC-Bayesian.}
In this paper, we use the PAC-Bayesian bound presented in \citet{catoni2007pac} which enjoys a tighter $O(\KL{Q}{P}/n)$ rate comparing to the traditional $O(\sqrt{\KL{Q}{P}/n})$ bound, but with a slightly larger constant factor on the empirical error. We restate their bound as follows.
\begin{lemma}[Catoni's Bound] (see e.g., \citet{lever2013tighter})\label{lem:catoni} For any prior distribution $P$ independent of the training set $S$, any $\delta \in (0, 1)$, and any $\eta>0$, the following bound holds w.p. $\geq 1-\delta$ over $S\sim \D^n$:
\begin{align}
\label{eq:pacbayesian}
\E_{W\sim Q}[\risk(W,\D)] &\leq \eta C_{\eta} \E_{W\sim Q}[\risk(W,S)] + C_\eta \cdot \frac{\KL{Q}{P} + \ln(1/\delta)}{n} \quad (\forall Q),
\end{align}
where $C_\eta = \frac{1}{1-e^{-\eta}}$ is an absolute constant.
\end{lemma}
\vspace{-0.35cm}
\paragraph{Concentration inequality.} We use the following variant of McDiramid inequality (Lemma~\ref{lem:new-mcd}) to prove the concentration of cumulative gradient difference in Section~\ref{sec:ld}. The proof is deferred to Appendix~\ref{app:prelim}.
% \begin{definition}[Order-Independent]\label{def:order-ind} A function $\Phi:[n]^m \to \R^+$ is said to be order independent if and only if $\Phi(x_1,x_2,...,x_m) = \Phi(x_{\pi_1},x_{\pi_2},...,x_{\pi_m})$ holds for any input $X=(x_1,x_2,...,x_m)\in \Omega^m$ and any permutation $\pi \in \mathbb{S}_{m}$.
% \end{definition}
\begin{lemma}\label{lem:new-mcd} Suppose $\Phi:[n]^m \to \R^{+}$ is order-independent\footnote{$\Phi(j_1,...,j_m) = \Phi(j_{\pi_1},...,j_{\pi_m})$ holds for any input $J=(j_1,...,j_m)\in \Omega^m$ and any permutation $\pi \in \mathbb{S}_{m}$.} and $|\Phi(J) - \Phi(J')|\leq c$ holds for any adjacent $J,J'\in [n]^m$ satisfying $|J \cap J'| = m-1$\footnote{$J\cap J':=\{i\in[n] : i\in J \cap i \in J'\}.$}. Let $J$ be $m$ indices sampled uniformly from $[n]$ without replacement. Then $\Pr_{J}\left[\Phi(J) - \E_J[\Phi(J)] > \epsilon\right] \leq \exp(\frac{-2\epsilon^2}{m c^2})$.
\end{lemma}
\section{Data-Dependent PAC-Bayesian Bound}\label{sec:datapac}
The dominating term in the PAC-Bayesian bound \eqref{eq:pacbayesian}
is $ \KL{Q}{P}/n$, where $P$ is a prior distribution independent of the training dataset $S$. Typically, without knowing any information from $S$, the best possible bound for $\KL{Q}{P}$ we can hope is at least $\Theta(1)$ (it should not be a function of $n$ hence should not decrease with $n$). However, if we are allowed to see $m$ data points from $S$ when constructing our prior, we may produce better prediction on posterior $Q_S$. The following theorem enables us to use data-dependent prior in PAC-Bayesian bound. The proof is almost the same as Cantoni's original proof and we provide a proof for completeness in Appendix~\ref{app:data-pac}.

\begin{theorem}[Data-Dependent PAC-Bayesian]\label{thm:data-pac} Suppose $J$ is a random sequence including $m$ indices uniformly sampled from $[n]$ without replacement. For any $\delta \in (0, 1)$ and $\eta > 0$, we have w.p. $\geq 1-\delta$ over $S\sim \D^n$ and $J$:
\begin{align*}
\risk(Q,\D) \leq \eta C_{\eta}\risk(Q,S_I) + C_{\eta}\cdot \frac{\KL{Q}{P(S_J)} + \ln(1/\delta)}{n-m} \quad (\forall Q),
\end{align*}
where $I = [n]\backslash J$ is the set of indices not in $J$, $P(S_J)$ is the prior distribution only depending on the information of $S_J$ ($S_J$ is the subset of $S$ indexed by $J$), and $C_{\eta}:=\frac{1}{1-e^{-\eta}}$ is a constant.
\end{theorem}
{\bf Remarks.}
Note that the above bound holds regardless of whether $Q$ depends on $S$ or not. 
%Usually, we apply this bound to the posterior $Q_S$ defined by our learning algorithm such as GD or GLD.
Also note that the first term in the right hand side is $\risk(Q,S_I)$, not $\risk(Q,S)$ as in the usual generalization bounds.
We remark that for most of our learning algorithms that are independent of $J$ (i.e.,
changing $J$ does not change the output $Q$), by standard Chernoff-Hoeffding inequality, 
$\risk(Q, S_I)$ can be bounded by $\risk(Q,S) + O(1/\sqrt{n-m})$ with high probability over the randomness of $J$. 
For example, the update rules of GLD, SGLD and CLD are independent of $J$, hence $\risk(Q, S_I)$ can be replaced by $\risk(Q,S) +  O(1/\sqrt{n-m})$ in Theorem~\ref{thm:data-pac}.
However, we point out a subtle point that 
FGD (Algorithm~\ref{alg:fgd}) studied in this paper depends on $J$. 
It may be the case that by knowing $J$, FGD extracts more information from $S_J$ but not much from $S_I$, unintentionally making $\risk(Q, S_I)$ a validation error, rather than the training error as it should be. 
%Will it possible that FGD can only learn $S_J$ well while leave $S_I$ as a validation set? 
However, 
%for the algorithm FGD we study, it does not really discriminate samples from $J$ and from $I$.  
from our experiment (see Figure~\ref{fig:fgd-gd-mnist} and~\ref{fig:fsgd-sgd-cifar} in Appendix~\ref{app:exp-detail}, and Figure~\ref{fig:fsgd-cifar10-exp-a}), 
we can see that FGD is very close to GD and the $S_I$ error $\risk(W_T, S_I)$ is indeed close to the training error $\risk(W_T, S)$ and both are significantly smaller than the testing error $\risk(W_T, \D)$. So $\risk(Q, S_I)$ can be considered as a genuine training error in our study of FGD. 

\section{FGD and FSGD}\label{sec:fgd}
In this section, we study the generalization error of finite precision variants of gradient descent and stochastic gradient descent: 
Floored Gradient Descent (FGD) and Floored Stochastic Gradient Descent (FSGD). 

First we need to define the ``floor'' operation which is used in the definitions of FGD and FSGD.
\begin{definition}[Floor]\label{def:floor} For any vector $X \in \R^d$, let $Y=\floor(X)$ defined as:
\begin{align*}
    Y_i= \floor(X_i) =
     \lfloor X_i\rfloor \text{ if } X_i \geq 0, \quad 
    = - \lfloor -X_i\rfloor \text{ if } X_i < 0, \text{ for all } i\in [d].
\end{align*}
\end{definition}

\paragraph{FGD:}
The Floored Gradient Descent algorithm is formally defined in Algorithm~\ref{alg:fgd}, where $(\gamma_t)_{t\geq 0}$ and $(\eps_t)_{t \geq 0}$ are the step size and precision sequences, respectively.
For a subset $Z\subseteq S$, we write $\g f(W_{t-1}, Z):=\frac{1}{|Z|}\sum_{z\in Z} \g f(W_{t-1}, z)$.
Note that FGD can be viewed as gradient descent with given precision limit $\eps_t$. 
We can see if we ignore the floor operation or let $\eps_t$ approach 0, FGD reduces to the ordinary GD (see Appendix~\ref{app:fgdintro}).
We also study momentum FGD, in which 
the 5th line of Algorithm~\ref{alg:fgd} is replaced by 
$$
    W_t \gets W_{t-1} + \alpha \cdot \left(W_{t-1} - W_{t-2}\right) - g_2 - \eps_t\cdot \floor((g_1 - g_2)/\eps_t);
$$
Here $\alpha>0$ is a constant.
We remark that both FGD and its momentum version are deterministic algorithms. 
The following theorem provides the generalization error bound for 
both algorithms.

\begin{algorithm}[t]
\KwIn{Training dataset $S=(z_1,..,z_n)$. Index set $J$. }
\KwResult{Parameter $W_T\in \R^d$.}
\caption{Floored Gradient Descent (FGD)\label{alg:fgd}}
Initialize $W_0\gets w_0$\;
\For {$t:1\to T$} {
    $g_1\gets \gamma_t\g f(W_{t-1}, S)$\;
    $g_2\gets \gamma_t\g f(W_{t-1}, S_J)$\;
    $W_t \gets W_{t-1} - g_2 - \eps_t\cdot \floor((g_1 - g_2)/\eps_t)$\;
}
\end{algorithm}

\begin{theorem}\label{thm:fgd}
Suppose $J$ is a random sequence consisting of $m$ indices uniformly sampled from $[n]$ without replacement. Then for any $\delta \in (0,1)$, 
both FGD (Algorithm~\ref{alg:fgd}) and its momentum version satisfy the following generalization bound w.p. at least $1-\delta$ over $S\sim \D^n$ and $J$:
\begin{align*}
\risk(W_T,\D) &\leq \eta C_{\eta}\risk(W_T,S_I) + C_{\eta}\cdot \frac{\ln(1/\delta) + 3}{n-m}
+\frac{C_{\eta}\ln(dT)}{n-m}\sum_{t=1}^T\left( \frac{\gamma_t^2}{\eps_t^2}\norm{\gdiff_t}^2\right),
\end{align*}
where $d$ is the dimension of parameter space,  $I = [n]\backslash J$ is the set of indices not in $J$, $C_{\eta}:=\frac{1}{1-e^{-\eta}}$ is a constant, and $\gdiff_t:=\g f(W_{t-1},S)-\g f(W_{t-1},S_J)$.
\end{theorem}
\begin{proof}
We use Theorem~\ref{thm:data-pac} to prove our theorem
for the momentum version. 
The ordinary FGD is a special case of the momentum version 
with $\alpha=0$.
The key is to construct the prior distribution $P(S_J)$ such that $\KL{W_T}{P(S_J)}$ is tractable. Let $p$ be any real number in $(0,1/3)$. We first define a stochastic process
$\{W'_0,\ldots, W'_T\}$, by the following update rule ($W'_0:=w_0$): 
\begin{align*}
W'_t \gets W'_{t-1}+\alpha \cdot \left(W'_{t-1} - W'_{t-2}\right) - \gamma_t \g f(W'_{t-1}, S_J) - \eps_t\cdot \xi_t,
\end{align*}
where $\xi_t$ is a discrete random variable such that for all $(a_1,..,a_d)\in \mathbb{Z}^d$:
\[\Pr[\xi_t = (a_1,...,a_d)^\top]:= \left(\sum_{i=-\infty}^{\infty}p^{i^2}\right)^{-d}\exp\left(-\sum_{k=1}^d \ln(1/p) a_k^2\right).\]
It is easy to verify that the sum of the probabilities ($\sum_{a\in\mathbb{Z}^d}\Pr[\xi_t=a]$) equals to $1$. 
Note that $W'_t$ only depends on $S_J$.
We define $P(S_J)$ as the distribution of $W'_T$.

Recall that $W_0^t = (W_0,...,W_t)$ is the parameter sequence of FGD (Algorithm~\ref{alg:fgd}). Applying the chain rule of KL-divergence (Lemma~\ref{lem:chain-kl}), we have:
\begin{equation}\label{eq:fgd-kl-chain}
\begin{split}
\KL{W_T}{P(S_J)} &= \KL{W_T}{W'_T} 
\leq \KL{W_0^T}{{W'}_0^T}\\
&= \sum_{t=1}^T \E_{w\sim W_0^{t-1}}\left[\KL{W_t|W_0^{t-1}=w}{W'_t|{W'}_0^{t-1}=w}\right]\\
&= \sum_{t=1}^T \KL{W_t|W_0^{t-1}=W_0^{t-1}}{W'_t|{W'}_0^{t-1}=W_0^{t-1}}.\\
\end{split}
\end{equation}
The last equation holds because FGD is deterministic. Let $w = W_0^{t-1}$. The distribution of $W_t|W_0^{t-1}=w$ (where $w=(w_0,...,w_{t-1})$) is a point mass on
\[
w_{t-1} + \alpha \cdot \left(w_{t-1} - w_{t-2}\right) - \gamma_t\g f(w_{t-1},S_J) - \eps_t \cdot \floor\left(\frac{\gamma_t(\g f(w_{t-1}, S) - \g f(w_{t-1}, S_J))}{\eps_t}\right).
\]
Let vector $a = (a_1,\ldots, a_d)= \floor(\frac{\gamma_t}{\eps_t}(\g f(w_{t-1}, S) - \g f(w_{t-1}, S_J)))$.
By the definition of $W'_t$, we have 
\begin{align*}
&\KL{W_t|W_0^{t-1}=w}{W'_t|{W'}_0^{t-1}=w} = 1 \cdot \ln\left(1/\Pr\left[\xi_t = a\right]\right)\\
&= \ln\left(\Bigl(\sum_{i=-\infty}^{\infty}p^{i^2}\Bigr)^d\right)+\sum_{k=1}^d \ln(1/p) \cdot a_k^2.
\end{align*}
Since $|i| \leq i^2$ and $p\in(0,1/3)$, we have $\ln\left((\sum_{i=-\infty}^{\infty}p^{i^2})^d\right)$ is at most $d\ln \left(1 + 2\sum_{i=1}^{\infty}p^i\right)$. It can be further bounded by $d \ln \left(1 + 3p\right)$. Moreover, it can be bounded by $3dp$ as $\ln(1+x) \leq x$. Thus, the above KL-divergence can be bounded by $3dp +\sum_{k=1}^d \ln(1/p) a_k^2$. Recall that the $k$th entry of $a$ is $a_k:=\lfloor \frac{\gamma_t}{\eps_t}\cdot(\g_k f(w_{t-1},S)-\g_k f(w_{t-1},S_J))\rfloor$, which is less than or equal to $\frac{\gamma_t}{\eps_t}\cdot(\g_k f(w_{t-1},S)-\g_k f(w_{t-1},S_J))$. Therefore, we have
\begin{align*}
\KL{W_t|W_0^{t-1}=w}{W'_t|{W'}_0^{t-1}=w}\leq 3dp + \frac{\ln(1/p)\gamma_t^2}{\eps_t^2}\norm{\g f(w_{t-1},S)-\g f(w_{t-1},S_J)}_2^2.
\end{align*}
Plugging the above inequality into \eqref{eq:fgd-kl-chain}, we have
\begin{align*}
\KL{W_T}{P(S_J)} &\leq \sum_{t=1}^T\left(3dp + \frac{\ln(1/p)\gamma_t^2}{\eps_t^2}\norm{\g f(W_{t-1},S)-\g f(W_{t-1},S_J)}_2^2\right).
\end{align*}
We conclude our proof by plugging it into Theorem~\ref{thm:data-pac} (setting $p=1/(Td)$).
\end{proof}
\vspace{-0.35cm}
\paragraph{FSGD:}
We can use a similar approach to prove a generalization bound for Floored Stochastic Gradient Desent (FSGD). Formally, FSGD is identical to Algorithm~\ref{alg:fgd} except for the definitions of $g_1$ and $g_2$ replaced with:
\[g_1 \gets \g f(W_{t-1},S_{B_t}), \quad g_2 \gets \g f(W_{t-1}, S_{B_t \cap J}),\]
where $B_t \subseteq [n]$ is a random batch independent of $S,J$ and $W_0^{t-1}$.
{Formally, each $B_t$ is a set including $b$ indices uniformly sampled from $[n]$ without replacement.}
The following theorem provides a generalization bound for FSGD. The proof can be found in Appendix~\ref{app:fgd}.
\begin{theorem}\label{thm:fsgd}
Suppose $J$ is a random sequence consisting of $m$ indices uniformly sampled from $[n]$ without replacement. Then for any $\delta \in (0,1), \eps\in(0,1)$,
FSGD satisfies the following generalization bound: w.p. at least $1-\delta$ over $S\sim \D^n$ and $J$:
\begin{align*}
\risk(W_T,\D) &\leq \eta C_{\eta}\risk(W_T,S_I) + C_{\eta}\cdot \frac{\ln(1/\delta) + 3}{n-m} +\frac{C_{\eta}\ln(dT)}{n-m}\E_{B_0^T}\left[\sum_{t=1}^T\frac{\gamma_t^2}{\eps_t^2}\norm{\gdiff_t}^2\right],
\end{align*}
where $d$ is the dimension of parameter space,  $I = [n]\backslash J$, $C_{\eta}:=\frac{1}{1-e^{-\eta}}$ is a constant, and $\gdiff_t:= f(W_{t-1},S_{B_t})-\g f(W_{t-1},S_{J\cap B_t})$. 
\end{theorem}

\section{Gradient Langevin Dynamics}\label{sec:ld}
\vspace{-0.2cm}
In this section, we present new
generalization bounds 
for Gradient Langevin Dynamics (GLD)
and Stochastic Gradient Langevin Dynamics (SGLD)
based on Theorem~\ref{thm:data-pac}.

\paragraph{Gradient Langevin Dynamics (GLD):}
The GLD algorithm can be viewed as gradient descent plus a Gaussian noise. Formally, for a given training set $S\sim \D^n$, the update rule of GLD is defined as follows:
\begin{align*}\label{eq:gld}
W_{t+1} \gets W_{t} -  \frac{\gamma_{t+1}}{n}\sum_{z\in S} \g f(W_t, z) + \sigma_{t+1} \N(0,I_d), \tag{GLD}
\end{align*}
Here the gradient $\g f(W_t, z)$ can be replaced with any gradient-like vector such as a clipped gradient.
The output of \ref{eq:gld} is the last step parameter $W_T$ or some function of the whole training trajectory $W_0^T$ (e.g., the average of the suffix $\frac{1}{K}\sum_{t=T-K}^T W_{t}$).

We still use the data-dependent PAC-Bayesian framework (Theorem~\ref{thm:data-pac}) to prove the generalization bound for GLD. Unlike FGD (Algorithm~\ref{alg:fgd}), GLD is independent of the prior indices $J$, which enables us to prove the following  concentration bound (Lemma~\ref{lem:gld-grad-con}) for the gradient difference. The proof is based on Lemma~\ref{lem:new-mcd}, which is postponed to Appendix~\ref{app:gld}.

\begin{lemma}\label{lem:gld-grad-con} Let $S = (z_1,...z_n)$ be any fixed training set. $J$ is a random sequence including $m$ indices uniformly sampled from $[n]$ without replacement, and $W=(W_0,...,W_T)$ is any random sequence independent of $J$. Then the following bound holds with probability at least $1-\delta$ over the randomness of $J$:
\[\E_W\left[\sum_{t=1}^T\frac{\gamma_t^2}{\sigma_t^2}\norm{\g f(W_{t-1},S) - \g f(W_{t-1},S_J)}^2\right] \leq \frac{C_{\delta}}{m}\E_W\left[\sum_{t=1}^T\frac{\gamma_t^2}{\sigma_t^2}L(W_{t-1})^2\right],\]
where $C_{\delta} = 4+ 2\ln(1/\delta) + 5.66\sqrt{\ln(1/\delta)}$, and $L(w) = \max_{i\in[n]}\norm{\g f(w,z_i)}$.
\end{lemma}
Now we are ready to present our main results. The proofs can be found in Appendix~\ref{app:gld}.
\vspace{-0.1cm}
\begin{theorem}\label{thm:gld-bound}
Suppose $J$ is a random sequence consisting of $m$ indices uniformly sampled from $[n]$ without replacement. Let $W_T$ be the output of \ref{eq:gld}. Then for any $\delta \in (0, \frac{1}{2})$ and $\eta > 0$, we have w.p. $\geq 1-2\delta$ over $S\sim \D^n$ and $J$, the following holds ($L(w):=\max_{z \in S}\norm{f(w,z)}$):
\begin{align*}
\risk(W_T,\D) \leq \eta C_{\eta} \risk(W_T,S_{I}) &+  \frac{C_{\eta}\ln(1/\delta)}{n-m} + \frac{C_{\eta}C_{\delta} }{2(n-m)m}\E_{W_0^T}\left[\sum_{t=1}^T\frac{\gamma_t^2}{\sigma_t^2}L(W_{t-1})^2\right],
\end{align*}
where  $C_{\delta} = 4+ 2\ln(1/\delta) + 5.66\sqrt{\ln(1/\delta)}$, $I = [n]\backslash J$ and $C_{\eta} = \frac{1}{1-e^{-\eta}}$.
\end{theorem}

\paragraph{Stochastic Gradient Langevin Dynamics (SGLD):}
For a given training data set $S$, the update rule of SGLD is defined as:
\begin{align*}\label{eq:sgld}
W_{t+1} \gets W_{t} - \gamma_{t+1} \g f(W_{t}, S_{B_t}) + \sigma_{t+1} \N(0,I_d), \tag{SGLD}
\end{align*}
where $B_t \sim \uni([n])^b$ is the mini-batch of size $b$ at step $t$. 
Note that $B_t$ is a sequence instead of a set, thus it may include duplicate elements. 
%The output of SGLD is the whole training sequence $W = (W_0,...,W_T)$. 
Similar to the analysis of GLD, we can prove the following bound for SGLD.

\begin{theorem}\label{thm:sgld-bound}
Let $W_T$ be the output of \ref{eq:sgld} when the training set is $S$, and $J$ be a random sequence with $m$ indices uniformly sampled from $[n]$ without replacement. For any $\delta \in (0, 1)$ and $m\geq 1$, we have w.p. $\geq 1-2\delta$ over $S\sim \D^n$ and $J$, the following holds:
\begin{align*}
\risk(W_T,\D) \leq \eta C_{\eta} \risk(W_T,S_I) +  \frac{C_{\eta}\ln(1/\delta)}{n-m} +
\frac{C_{\eta}}{n-m}\left(\frac{4}{b} + \frac{C_{\delta}}{2m}\right) \E_{W_0^T}\left[\sum_{t=1}^T\frac{\gamma_t^2}{\sigma_t^2}L(W_{t-1})^2\right],
\end{align*}
where $L(w):=\max_{z \in S}\norm{f(w,z)}$, $C_{\delta} = 4+ 2\ln(1/\delta) + 5.66\sqrt{\ln(1/\delta)}$, $C_{\eta}=\frac{1}{1-e^{-\eta}}$, $b$ is the batch size, and $I = [n]\backslash J$.
\end{theorem}

\begin{remark} If the gradient norm is bounded\footnote{$\norm{\g f(w,z)} \leq L$ holds for all $w,z$} and we use a decaying learning rate schedule such as $\gamma_t \propto O(1/t)$, 
then the summation in our bound converges. 
Hence, under such a learning rate schedule, Theorem~\ref{thm:gld-bound} and ~\ref{thm:sgld-bound} imply the following test error bound for GLD or SGLD: $\risk(W_T,\D) \leq \eta C_{\eta}\risk(W_T,S_I)+\widetilde{O}(\frac{1}{n-m})$ which is independent of $T$,
where $\widetilde{O}$ hides some logarithmic factors.
\end{remark}

\begin{figure}[t]
    \centering
    \subcaptionbox{entire training path\label{fig:fgd-mnist-exp-a}}
      {\includegraphics[width=0.3\linewidth]{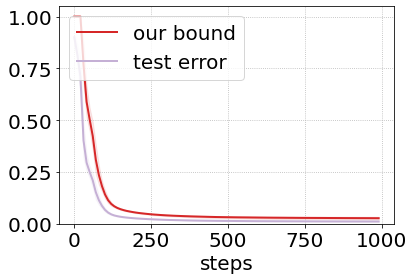}}
    \subcaptionbox{later stage\label{fig:fgd-mnist-exp-b}}
      {\includegraphics[width=0.3\linewidth]{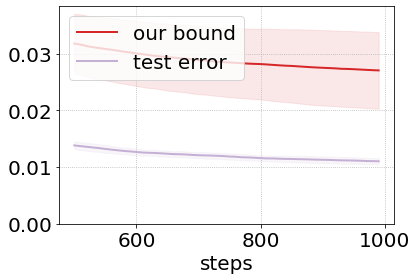}}
    \subcaptionbox{gradient difference\label{fig:fgd-mnist-exp-c}}
      {\includegraphics[width=0.3\linewidth]{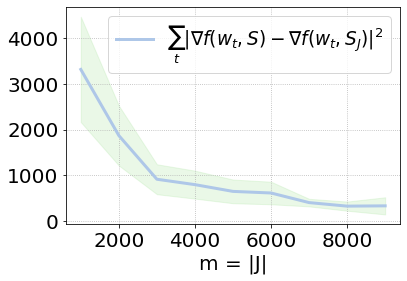}}
    \caption{MNIST + CNN + FGD. In (a) and (b), we plot the true test error
    and our bound (Theorem~\ref{thm:fgd} with $\eta=1.5,\delta=0.1$). 
    In (c), we  show how cumulative gradient difference decreases as $m$
    (the size of $J$) increases.}
\end{figure}
\begin{figure}[t]
    \centering
    \subcaptionbox{training and test errors (FSGD)\label{fig:fsgd-cifar10-exp-a}}
      {\includegraphics[width=0.3\linewidth]{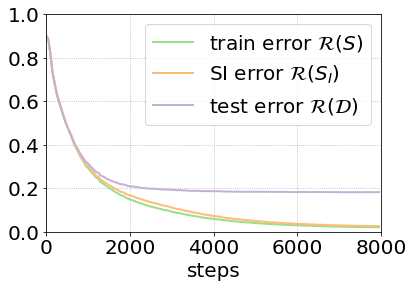}}
    % \subcaptionbox{bound\label{fig:fsgd-cifar10-exp-b}}
    %   {\includegraphics[width=0.22\linewidth]{figs/cifar10/cifar_1.png}}
    \subcaptionbox{our bound\label{fig:fsgd-cifar10-exp-c}}
      {\includegraphics[width=0.3\linewidth]{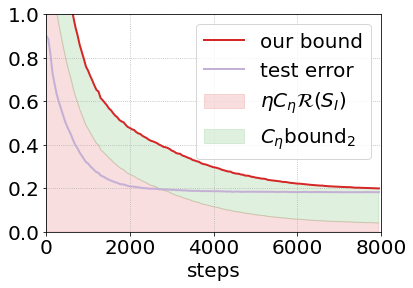}}
    \subcaptionbox{gradient difference \label{fig:fsgd-cifar10-exp-d}}
      {\includegraphics[width=0.3\linewidth]{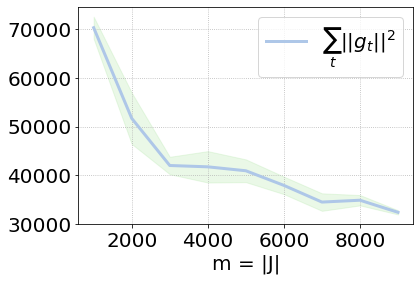}}
    \caption{CIFAR10 + SimpleNet + FSGD. In (a), we plot $\risk(W_T,S_I)$,
    $\risk(W_T, S)$ and the test error. We can see that 
    $\risk(W_T,S_I)$ is very close to $\risk(W_T, S)$. 
    In (b), we plot our theoretical bound (Theorem~\ref{thm:fsgd} with $\eta=2,\delta=0.1$). The red part corresponds to the first term of our bound (the empirical risk) and the green part corresponds to the rest. 
    The last step test error and our bound are $0.18$ and $\mathbf{0.198}$, respectively. In (c), we show how cumulative gradient difference decreases as $m$ (the size of $J$) increases.}
\end{figure}
\section{Experiment}\label{sec:exp}

In this section, we conduct experiments for FGD and FSGD on MNIST~\citep{lecun1998gradient} and CIFAR10~\citep{krizhevsky2009learning} to investigate the 
the optimization and generalization properties of FGD and FSGD,
and the numerical closeness between
our theoretical bounds and true test errors. 
Due to space limit, the detailed experimental setting 
and some additional experimental results can be found in Appendix~\ref{app:exp-detail}.

\paragraph{FGD/FSGD vs GD/SGD.} 
We first demonstrate that the training and testing curves of FGD and GD are nearly identical (we choose precision level $\eps=0.005$ or $0.004$). 
We also show that the same is true for FSGD vs SGD. 
Due to space limit, the figures are presented in Appendix~\ref{app:exp-detail} (Figure~\ref{fig:fgd-gd-mnist} and~\ref{fig:fsgd-sgd-cifar}).

\paragraph{Non-vacuous bounds.} For MNIST, we train a CNN ($d=1.4\cdot10^6$) by FGD with $\gamma_t=0.005\cdot 0.9^{\lfloor \frac{t}{150}\rfloor}$ and $\eps_t=0.005$ and momentum $\alpha=0.9$). The size $m = |J|$ is set to $n/2 = 30000$. As shown in Figure~\ref{fig:fgd-mnist-exp-a} and~\ref{fig:fgd-mnist-exp-b}, our bound (Theorem~\ref{thm:fgd} with $\eta=1.5, \delta=0.1$) tracks the testing error closely. At step $T=990$, our bound is $\mathbf{0.026}$ while the testing error is $0.011$. This is non-vacuous and tighter than best known $11\%$ MNIST bound reported in \citet{dziugaite2021role}. For CIFAR10, we train a SimpleNet~\citep{hasanpour2016lets} without BatchNorm and Dropout. The number of parameters $d$ is nearly $18\cdot 10^6$. We use FSGD to train our model. The learning rate $\gamma_t$ is set to $0.001\cdot 0.9^{\lfloor t/200\rfloor}$, the precision $\eps_t$ is set to $0.004$, and the momentum $\alpha$ is set to $0.99$. The batch size is $2000$. $m = |J|$ is set to $n/5 = 10000$. The result is shown in Figure~\ref{fig:fsgd-cifar10-exp-c}. We stop training at step $t=8000$ 
when the testing error is and $0.18$. At that time, our testing error bound is $\mathbf{0.198}$ which is non-vacuous and tighter than best known ${0.23}$ CIFAR10 bound reported in \citet{dziugaite2021role}.

\paragraph{Decrease of the gradient difference.}
Intuitively, the cumulative squared norm of gradient difference $\gdiff_t:=\g f(W_t, S) - \g f(W_t, S_J)$ should decrease as 
$m=|J|$ increases. Although we cannot prove a concentration like Lemma~\ref{lem:gld-grad-con} (i.e., $\norm{\gdiff_t}^2$ scales as $O(1/m)$), 
we can still observe that $\norm{\gdiff_t}^2$ decreases when $m$ increases. 
The results are depicted in Figure~\ref{fig:fgd-mnist-exp-c} and Figure~\ref{fig:fsgd-cifar10-exp-d}.

\paragraph{Random labels.}
We conduct the random label experiment designed in \citet{ZhangBHRV17}.
Our theoretical bounds can distinguish the datasets with different portion ($p$) of random labels. See Appendix~\ref{app:exp-detail}.
%As shown in Figure~\ref{fig:fgd-mnist-random} and ~\ref{fig:fsgd-cifar-random}, our bound correctly depicted the relative order of the true testing errors.

\section{Conclusion}

In this paper, we prove new generalization bounds for several gradient-based methods with either discrete or continuous noises based on carefully constructed data-dependent priors. 
Recall that FGD requires to compute the gradient difference for technical reasons. It would be more natural and desirable if we only need to compute the full gradient and rounded to the nearest grid point. An intriguing future direction is to free FGD/FSGD from the dependence of the prior subset $J$ so that we can apply the concentration on the gradient difference to obtain a tighter bound.
Of course, a major further direction is to obtain similar generalization bounds for vanilla GD and SGD, which remains to be an important open problem in this line of work.
Our technique can be useful for handling deterministic algorithms and discrete noises, but it seems that new technical ideas or assumptions are needed for tackling GD or SGD.

\section{Acknowledgements}
The authors would like to thank the anonymous reviewers for their constructive comments. The authors are supported in part by the National Natural Science Foundation of China Grant 62161146004, Turing AI Institute of Nanjing and Xi'an Institute for Interdisciplinary Information Core Technology. 
\bibliography{main}
\newpage
\appendix
\section{Floored Gradient Descent}
\label{app:fgdintro}
FGD is a finite precision variant of GD. It decomposes the full gradient $g_1:=\gamma_t \nabla f(W_t,S)$ (used in GD) into a sum of two parts, $g_2:=\gamma_t \nabla f(W_t, S_J)$ and $\Delta g:=\gamma_t (g_1-g_2)=\gamma_t\nabla f(W_t, S) - \gamma_t\nabla f(W_t,S_J)$, where 
$J\subseteq [n]$ is a subset fixed before training and
$S_J$ is the subset of training data corresponding to index set $J$.
Note that $S_J$ is the ``prior" dataset, rather than a mini-batch. 
Then we reduce the precision of $\Delta g$ to $\varepsilon_t$ by applying a floor-operation $\Delta' g:=\varepsilon_t \mathrm{floor}(\Delta g/\varepsilon_t)$. Hence, $g_2 + \Delta' g$ can be viewed as an approximation of full gradient $g_1=g_2 + \Delta g$. 

It is easy to see that if we ignore the floor operation or 
when $\varepsilon_t$ goes to $0$, FGD becomes GD. 
More concretely, recall that the update rule of FGD is (WLOG let $\gamma_t=1$):
$$
W_t \gets W_{t-1} - \nabla f(W_{t-1}, S_{J}) - \varepsilon_t\cdot \mathrm{floor}\left(\frac{\nabla f(W_{t-1}, S) - \nabla f(W_{t-1}, S_{J})}{\varepsilon_t}\right).
$$

(1) If we ignore the floor operation in the last term, the equation becomes
\begin{align*}
W_t & \gets W_{t-1} - \nabla f(W_{t-1}, S_{J}) - 
\varepsilon_t\cdot \left(\frac{\nabla f(W_{t-1}, S) - \nabla f(W_{t-1}, S_{J})}{\varepsilon_t}\right) \\
& =W_{t-1} - \nabla f(W_{t-1}, S_{J}) - (\nabla f(W_{t-1}, S) - \nabla f(W_{t-1}, S_{J})) \\
& =W_{t-1} - \nabla f(W_{t-1}, S).
\end{align*}
 
 (2) When $\varepsilon \rightarrow 0$,
 it is easy to see that 
$\lim_{\varepsilon \rightarrow 0} \varepsilon \cdot \mathrm{floor}(x/\varepsilon) = x$. Again, the FGD update rule reduces to
\begin{align*}
W_{t}  & \leftarrow W_{t-1} - \nabla f(W_{t-1}, S_{J}) - (\nabla f(W_{t-1}, S) - \nabla f(W_{t-1}, S_{J})) \\
& =W_{t-1} - \nabla f(W_{t-1}, S).
\end{align*}
The contribution of $\nabla f(W_{t-1}, S_J)$ is canceled out, and again 
FGD becomes GD.

\section{Comparison with Existing Work}\label{app:compare}

\subsection{Numerical bounds on MNIST and CIFAR10}

    The following table (Table~\ref{tab:compare}) lists some of existing theoretical test bounds
    for MNIST and CIFAR10.
    Typically, for both datasets, one can fit the training set very well
    and obtain a near zero training error.
    Hence, the generalization error bound is quite close to the 
    theoretical test error bound.
    There is no numerical experiment in \citet{mou2018generalization}
    and the number is excerpted from \cite{negrea2019information}.
    The first few rows are excerpted from \cite{nagarajan2019generalization}.
    
    \begin{table}[h]
    \centering
    \begin{tabular}{|c|c|c|c|c|}
     \hline
     Reference & Algorithms & Approach & MNIST & CIFAR10\\
     \hline
     \citet{harvey2017nearly} & Any & VC-dim &  large &  large\\
     \hline
     \citet{bartlett2017spectrally} & Any & Margin-based &  large &  large\\
     \hline
     \citet{golowich2018size} & Any & Rademacher Complexity &  large &  large\\
     \hline
    %  \citet{nagarajan2019generalization} & arbitrary & norm-based & very large & very large\\
    %  \cline{1-2}
     \citet{arora2018stronger} & Any & Compression& - &  large\\
     \hline
     \citet{dziugaite2017computing} & OPT-PAC & PAC-Bayes& 0.161 & - \\
     \hline 
     \citet{mou2018generalization} & SGLD & PAC-Bayes & $\approx$ 1.2 & - \\
     \hline
     \citet{li2019generalization} & GLD/SGLD & Bayes-Stability & $\approx$ 0.2 & $\approx$ 207 \\
     \hline
     \citet{zantedeschi2021learning} & Majority Votes & PAC-Bayes & $\approx$ 0.45 & - \\
     \hline
     \citet{ZhouVAAO19} & Any & PAC-Bayes& 0.46 & -\\
     \hline
     \citet{negrea2019information} & SGLD & Information theory & 0.21 & 41.13\\
     \hline
     \citet{haghifam2020sharpened} & SGLD & Information theory & $\approx$ 0.15 & $\approx$ 0.72\\
     \hline
     \citet{dziugaite2021role} & OPT-PAC & PAC-Bayes & 0.11 & 0.23 \\
     \hline
     \bf{Our bound} & FGD/FSGD & PAC-Bayes& \bf{0.026} & \bf{0.198}\\
     \hline
    \end{tabular}
    
    \vspace{0.2cm}
    \caption{Comparison with existing theoretical upper bounds of 
    test errors on MNIST and CIFAR10. 
    ``Any'' means that the bound only depends on the trained network, not the 
    training algorithm. ``OPT-PAC" means that the work optimizes a different 
    loss function that corresponds to the PAC-Bayesian bound.
    ``large'' means that the bound is far greater than 1 
    and ``-'' indicates the bound is not reported in that paper.
    \label{tab:compare}}
    \end{table}

\subsection{Comparison with Existing GLD/SGLD bounds}

We compare our bounds for GLD/SGLD (Theorem~\ref{thm:gld-bound} and~\ref{thm:sgld-bound}) with existing GLD/SGLD generalization bounds in prior work.
For GLD, \citet{mou2018generalization} provide 
a generalization bound in expectation based on the uniform stability framework,
which is of rate $O(\frac{L\sqrt{T}}{n})$, where $L$ is the global Lipschitz constant (ignoring the factors depending on $\gamma_t$ and $\sigma_t$).
Their bound can be tighten to $O(\frac{1}{n}\sqrt{\sum_{t=1}^T L^2_t})$ where $L_t$ is the gradient norm
at time $t$ (which is always less than $L$) ~\citep{li2019generalization}.
Their bounds can be converted to high probability bound with an additional factor $O(1/\sqrt{n})$ 
using the technique developed in~\citep{feldman2019high}. 
Bounds of similar orders have been also obtained through information theory 
\citep{wang2021analyzing, haghifam2020sharpened} and differential privacy \citep{wu2021generalization}.
Note that the main term in our bound is of order $O(\frac{1}{n^2}\sum_{t=1}^T L^2_t)$, which is quadratically better  than theirs if the bound is in $(0,1)$.
For SGLD, \citet{mou2018generalization} and \citet{li2019generalization} obtained similar bounds, but requires the assumption that the learning rate should be of order $O(1/L)$.
Our bound for SGLD does not require such assumption and is more favorable for large minibatch size $b$.
For small value of $b$ (say $b=O(1)$), our bound can be worse.

Another closely related work is~\citet{negrea2019information}.
They also use a data-dependent prior
and present an in-expectation bound based on information theory. 
They introduce a quantity called ``incoherence'' $\norm{\xi_t}$ that is defined somewhat similar to $\|\gdiff_t\|$
(it is also the norm of the different between two gradients defined by two different subsets of samples). They obtained an $O(\sqrt{\frac{1}{n-m}\sum_{t=1}^T\frac{\gamma_t^2}{\sigma_t^2}\E[\norm{\xi_t}^2]})$ bound for SGLD. By taking expectation, it can be further bounded by $O(\sqrt{\frac{1}{n}\sum_{t=1}^T(\frac{1}{b} + \frac{(n-m)}{n m})\frac{\gamma_t^2}{\sigma_t^2}V_t})$, where $V_t$ is a quantity about the same size as the variance of training gradients.
In fact, they obtain a worst case $O(\frac{(n-m)^2}{n^2}\sum_{t=1}^T\frac{L^2\gamma_t^2}{\sigma_t^2})$ bound for $\KL{Q}{P(S_J)}$, and if we plug it into Catoni's PAC-Bayesian bound (Theorem~\ref{thm:data-pac}), one can 
obtain an $O(\frac{1}{n-m} + \frac{(n-m)}{n^2}\sum_{t=1}^T\frac{\gamma_t^2}{\sigma_t^2}L^2)$ high probability bound (see Theorem~\ref{thm:negre} below for details).
To make the 2nd term in their bound have the same $O(1/n^2)$ rate as ours, one needs to set $n-m=O(1)$  which would result in a large first term $\frac{1}{n-m}=\Omega(1)$. 
Moreover, our construction of data-dependent prior $P(S_J)$ is also very different from theirs. Their idea is to use the gradients in $S_J$ to cancel out the gradients in $S$ while ours is based on the property that the mean gradient on $S_J$ is concentrated around the mean gradient of the whole dataset $S$.

{
\begin{theorem}\label{thm:negre} Suppose the loss $f$ is $L$-Lipschitz (i.e., $\norm{\g f(w,z)} \leq L$ holds for all $w,z$). Then for \ref{eq:gld}, we have the following bound holds w.p. at least $1-\delta$ over the randomness of $J$ and $S\sim \D^n$:
\begin{align*}
   \risk(W_T, \D) \leq \eta C_{\eta}\risk(W_T,S_{[n]\backslash J}) + O\left(\frac{\ln(1/\delta)}{n-m} + \frac{(n-m)}{n^2}\sum_{t=1}^T\frac{\gamma_t^2}{\sigma_t^2}L^2\right), 
\end{align*}
where $C_{\eta}=\frac{1}{1-e^{-\eta}}$.
\end{theorem}
\begin{proof}
Using a data-dependent prior $P(S_J)$ defined in \citet[Section 3.1.1]{negrea2019information}, one can obtain an $O(\frac{(n-m)^2}{n^2}\sum_{t=1}^T\frac{\gamma_t^2}{\sigma_t^2}L^2)$ bound for $\KL{W_T}{P(S_J)}$ (see \citet[Section 3.1.1]{negrea2019information} for details). We conclude the proof by plugging it into Theorem~\ref{thm:data-pac}. \end{proof}

} 

\citet{mou2018generalization} also obtain a high probability PAC-Bayesian bound 
of rate $O(\sqrt{\frac{1}{n}\sum_{t=1}^T e^{-r_t}L_t^2})$ if there is an $\ell_2$-regularization in the loss (ignoring other factors depending on $\gamma_t$ and $\sigma_t$).
Here $e^{-r_t}<1$ is a decay factor depending on the regularization coefficient.
There is a similar decay factor in \citet{wang2021analyzing}'s bound (the fact comes from strong data processing inequalities). Note that without $\ell_2$-regularization, there is no such decay factor, and \citet{mou2018generalization}'s bound becomes $O(\sqrt{T/n})$, which is looser than ours. 
From technical perspective, they use Fokker Planck equation to track the time derivative of KL 
and Logarithmic Sobolev inequality to related KL with Fisher information. We also use these tools
for our generalization bound of Continuous Langevin dynamics (CLD) (see Appendix~\ref{sec:cld}), but the general proof idea is very different. 
\section{Omitted Proofs in
Section~\ref{sec:prelim}}\label{app:prelim}
In this section, we are going to prove a generalized McDiarmid's inequality (Lemma~\ref{lem:new-mcd}). To avoid frequently writing long expression $\Phi(j_1,j_2,...,j_i,J_{i+1},...,J_m)$, we briefly denoted it as $\Phi(j_1^i,J_{i+1}^m)$, where $j_l^r$ is a abbreviation of sequence $ (j_l, j_{l+1},...,j_r)$. Before proving it, we need the following lemma:
\begin{lemma}[Theorem~5.3 in \citet{dubhashi2009concentration}]\label{lem:non-iid} Let $J = (J_1,...,J_m)$ be any random sequence and $\Phi$ be a function of $J$. If for any $i\in [m]$ and any fixed $j_{1}^i = (j_1,..,j_i)$, it satisfies
\[\left|\E_{J_{i+1}^m}[\Phi(j_1^i,J_{i+1}^m)|J_1^i=j_1^i] - \E_{J_{i}^m}[\Phi(j_1^{i-1},J_{i}^m)|J_1^{i-1}=j_1^{i-1}]\right| \leq c_i.\]
Then, the following inequality holds: 
\[\Pr_J\left[\Phi(J) - \E_J[\Phi(J)] > \epsilon \right] \leq \exp\left(-\frac{2\epsilon^2}{\sum_{i=1}^m c_i^2}\right).\]
\end{lemma}
Now we are able to prove our result.
\begin{clemma}{\bf \ref{lem:new-mcd}}\, Suppose $\Phi:[n]^m \to \R^{+}$ is order-independent and $|\Phi(J) - \Phi(J')|\leq c$ holds for any adjacent $J,J'\in [n]^m$ satisfying $|J \cap J'| = m-1$. Let $J$ be $m$ indices sampled uniformly from $[n]$ without replacement. Then $\Pr_{J}\left[\Phi(J) - \E_J[\Phi(J)] > \epsilon\right] \leq \exp(\frac{-2\epsilon^2}{m c^2})$.
\end{clemma}
\begin{proof}
It suffices to verify that the conditions in Lemma~\ref{lem:non-iid} are satisfied. For any fixed $j_1^i = (j_1,...,j_i)$, let 
$X$ and $Y$ be two independent random variables with distribution equal to $J | J_1^i = j_1^i$ and 
$J | J_1^{i-1} = j_1^{i-1}$, respectively.

The goal is to find an upper bound $c_i$ for $|\E[\Phi(X)] - \E[\Phi(Y)]|$. 
We distinguish the following disjoint events.
\begin{itemize}
    \item $\cond_1: Y_i = j_i$.\\ Conditioning on this, we have $X_{i+1}^m$ and $Y_{i+1}^m$ share the same distribution. Notice that $X_1^i = Y_1^i$. Thus, we have:
    \[|\E[\Phi(X) - \Phi(Y)\,\,\,|\,\,\cond_1]| = 0.\]
    \item $\cond_2: Y_i \neq j_i\quad\bigcap\quad Y_i \notin X_{i+1}^m \quad\bigcap\quad j_i \notin Y_{i+1}^m$.\\ Conditioning on this, both suffixes $X_{i+1}^m$ and $Y_{i+1}^m$ are sampled from $[n]\backslash (j_1^i \cup J'_i)$. Thus their distributions are identical. We have
    \begin{align*}
    |\E[\Phi(X)|\cond_1] &- \E[\Phi(Y)|\cond_2]| \\
    &\leq \E_{j_{i+1}^m\sim X_{i+1}^m}[|\Phi(j_1^{i-1},j_i, j_{i+1}^m) - \Phi(j_1^{i-1},Y_i, j_{i+1}^m)|]\\
    &\leq c. \tag{Assumption}
    \end{align*}
    \item $\cond_3:Y_i \neq j_i\quad\bigcap\quad Y_i \notin X_{i+1}^m\quad\bigcap\quad j_i \in Y_{i+1}^m$.\\ Without loss of generality, we assume $Y_{i+1} = j_i$. Then $X_{i+1}^{m-1}$ and $Y_{i+2}^m$ share the same distribution. Moreover, we have $\set(X) \bigcap \set(Y) = m - 1$. The only different pair is $(X_m,Y_i)$. Since $\Phi$ is order-independent, we have
    \begin{align*}
    |\E[\Phi(X)|\cond_3] - \E[\Phi(Y)|\cond_3]|\leq c.
    \end{align*}
    \item $\cond_4: Y_i \neq j_i\quad\bigcap\quad Y_i \in X_{i+1}^m \quad\bigcap\quad j_i \notin Y_{i+1}^m$.\\ Similar to the previous situation, we can prove 
    \[|\E[\Phi(X)|\cond_4] - \E[\Phi(Y)|\cond_4]|\leq c.\]
    \item $\cond_5: Y_i \neq j_i\quad\bigcap\quad Y_i \in X_{i+1}^m \quad\bigcap\quad j_i \in Y_{i+1}^m$.\\ Without loss of generality, we assume $Y_{i+1}=j_i$ and $X_{i+1}=Y_i$. Then $X_{i+2}^m$ and $Y_{i+2}^m$ have the same distribution. It further implies
    \[|\E[\Phi(X)|\cond_5] - \E[\Phi(Y)|\cond_5]| = 0.\]
\end{itemize}
Putting these together, we have
\begin{align*}
|\E[\Phi(X)] - \E[\Phi(Y)]| &\leq \sum_{k=1}^5\Pr[\cond_k]\cdot \left|\E[\Phi(X)|\cond_k] - \E[\Phi(Y)|\cond_k]\right|\\
&\leq c\cdot (\Pr[\cond_2 \bigcup \cond_3 \bigcup \cond_4])\\
&= c\cdot (\Pr[Y_i\neq j_i] - \Pr[\cond_5])\\
&= c\cdot \left(\frac{n-i}{n-i+1} - \frac{n-i}{n-i+1}\cdot \frac{1}{n-i} \cdot \frac{1}{n-i}\right)\\
&= c \cdot \frac{n-i-1}{n-i}.
\end{align*}
Therefore, we can apply Lemma~\ref{lem:non-iid} with $c_i = \frac{c(n-i-1)}{n-i}$ to obtain:
\begin{align*}
\Pr_{J}\left[\Phi(J) - \E_J[\Phi(J)] > \epsilon \right]&\leq \exp\left(\frac{-2\epsilon^2}{\sum_{i=1}^m c_i^2}\right)
= \exp\left(\frac{-2\epsilon^2}{c^2\sum_{i=1}^m \frac{(n-i-1)^2}{(n-i)^2}}\right).
\end{align*}
\end{proof}
\section{Omitted Proofs in Section~\ref{sec:datapac}}\label{app:data-pac}
\begin{ctheorem}{\bf\ref{thm:data-pac}}\, Suppose $J$ is a random sequence including $m$ indices uniformly sampled from $[n]$ without replacement. For any $\delta \in (0, 1)$ and $\eta > 0$, we have w.p. $\geq 1-\delta$ over $S\sim \D^n$ and $J$:
\begin{align*}
\risk(Q,\D) \leq \eta C_{\eta}\risk(Q,S_I) + C_{\eta}\cdot \frac{\KL{Q}{P(S_J)} + \ln(1/\delta)}{n-m} \quad (\forall Q),
\end{align*}
where $I = [n]\backslash J$, $P(S_J)$ is the prior distribution only depending on the information of $S_J$, and $C_{\eta}:=\frac{1}{1-e^{-\eta}}$ is a constant.
\end{ctheorem}
The proof is almost the same as ~\citet{catoni2007pac}[Theorem 1.2.6]. We prove it here for completeness.
\begin{proof}
For any $\lambda>0$, define $\Phi(x):= -\frac{n-m}{\lambda}\ln(1-(1-e^{-\frac{\lambda}{n-m}})x)$. To simplify notation, let $P_J$ denote $P(S_J)$. The goal is to prove
\begin{equation}\label{eq:catoni-mmt}
\E_{S\sim \D^n, J}\left[\exp\left(\sup_{Q\ll P_J}(\lambda(\Phi (\risk(Q,\D)) - \risk(Q,S_I)) - \KL{Q}{P_J}\right)\right] \leq 1.
\end{equation}
If \eqref{eq:catoni-mmt} holds, we can apply Markov inequality to prove our theorem. Because for any random variable satisfying $\E[e^X] \leq 1$, we have $\Pr[X > \ln({1}/{\delta})] = \Pr[e^X > \frac{1}{\delta}]$ which is less than or equal to $\frac{\E[e^X]}{1/\delta} \leq \delta$. It further implies with probability at least $1-\delta$:
\begin{align}\label{eq:catoni-high-prob}
\Phi(\risk(Q,\D)) \leq \risk(Q,S_I) + \frac{\KL{Q}{P_J} + \ln(1/\delta)}{\lambda}.
\end{align}
Note that $\Phi(x):(0,1)\to (0,1)$ is an increasing function whose inverse is given by 
\begin{align*}
\Phi^{-1}(x) = \frac{1-\exp(-\frac{x\lambda}{n-m})}{1-\exp(-\frac{\lambda}{n-m})}.
\end{align*}
We can compose $\Phi^{-1}$ to both sides of \eqref{eq:catoni-high-prob} and use the basic inequality $1-\exp(-x) \leq x\,\,\,(\forall x>0)$ to prove our theorem (let $\eta = \frac{\lambda}{n-m}$).

It remains to prove \eqref{eq:catoni-mmt}. First it is easy to verify that $\Phi(x)$ is convex when $x \in (0, 1)$. Hence, for any $Q$, we have the following holds by Jensen's inequality:
\begin{align*}
\Phi(\risk(Q,\D)) = \Phi(\E_{w\sim Q}\risk(w,\D)) &\leq \E_{w\sim Q}[\Phi(\risk(w,\D))].
\end{align*}
Define $h(w):=\lambda (\Phi(\risk(w,\D)) - \risk(w,S_I))$. Then, the LHS of \eqref{eq:catoni-mmt} is less than or equal to
\begin{align*}
    &\E_{S,J}\left[\exp\left(\sup_{Q \ll P_J}(\E_{w\sim Q}[h(w)] - \KL{Q}{P_J})\right)\right]\\
    &= \E_{S,J}\left[\exp \left( \ln \E_{w\sim P_J}[\exp(h(w))]\right)\right].
\end{align*}
The last equation is due to Donsker and Varadhan's variational formula~\citep{donsker1983asymptotic}. Moreover, since the $J$ and $S$ are independent, and $S_I$ and $S_J$ are independent, it can be rewritten as
\[\E_J\E_{S_J\sim \D^m}\E_{S_I\sim \D^{n-m}}\E_{w\sim P_J}[\exp(h(w))].\]
Note that $S_I$ is independent of $P_J$. We have the above formula is equal to
\begin{align*}
\E_J\E_{S_J\sim \D^m}\E_{w\sim P_J}\E_{S_I\sim \D^{n-m}}[\exp(h(w))].
\end{align*}
For any fixed $w, J, S_J$, let random sequence $S_I = (z_1,..,z_{n-m})$. We have
\begin{align*}
\E_{S_I\sim \D^{n-m}}[\exp(h(w))] &\leq \prod_{i=1}^{n-m}\E_{z_i \sim \D}\left[\exp\left(\frac{\lambda}{n-m}(\Phi(\risk(w,\D)) - \risk(w,z_i))\right)\right]\\
&= \prod_{i=1}^{n-m} 1.
\end{align*}
We give a detailed proof for the last equation. Suppose $i\in [n]$ is fixed. Let $b$ denote the random variable $\risk(w,z_i)$. Note that $b$ is a Bernoulli random variable with mean $q:=\risk(w,\D)$. Thus, we can directly compute the multiplier term:
\begin{align*}
    &\exp\left(\frac{\lambda}{n-m}(\Phi(\risk(w,\D))\right)\E_{z_i
    \sim \D}\left[\frac{1}{\exp\left(\frac{\lambda}{n-m}\risk(w,z_i)\right)}\right] \\
    &= \frac{1}{1-(1-e^{-\frac{\lambda}{n-m}})q}\cdot \left(q e^{-\frac{\lambda}{n-m} \cdot 1} + (1-q) e^{0}\right)\\
    &=1.
\end{align*}
\end{proof}

\section{Omitted Proofs in Section~\ref{sec:fgd}}\label{app:fgd}

\paragraph{FSGD Proofs.}
We formally define FSGD in Algorithm~\ref{alg:fsgd}.  The only difference is that we sample a mini-batch $B_t$ before each step. Recall that each $B_t$ is a set including $b$ indices uniformly sampled from $[n]$ without replacement.

\begin{algorithm}
\KwIn{Training dataset $S=(z_1,..,z_n)$. Index set $J$. Momentum coefficient $\alpha$.}
\KwResult{Parameter $W_T\in \R^d$.}
\caption{Floored Stochastic Gradient Descent (FSGD) \label{alg:fsgd}}
Initialize $W_0\gets w_0$\;
\For {$t:1\to T$} {
    $B_t \gets$ a random mini-batch with size $n_{\text{batch}}$\;
    $g_1\gets \gamma_t\g f(W_{t-1}, S_{B_t})$\;
    $g_2\gets \gamma_t\g f(W_{t-1}, S_{J\cap B_t})$\;
    $W_t \gets W_{0}^{t-1} + \alpha \cdot \left(W_{t-1} - W_{t-2}\right) - g_2 - \eps_t\cdot \floor((g_1 - g_2)/\eps_t)$\;
}
\end{algorithm}
\begin{ctheorem}{\bf \ref{thm:fsgd}}\,
Suppose $J$ is a random sequence consisting of $m$ indices uniformly sampled from $[n]$ without replacement. Then for any $\delta \in (0,1), \eps\in(0,1)$,
FSGD satisfies the following generalization bound: w.p. at least $1-\delta$ over $S\sim \D^n$ and $J$:
\begin{align*}
\risk(W_T,\D) &\leq \eta C_{\eta}\risk(W_T,S_I) + C_{\eta}\cdot \frac{\ln(1/\delta) + 3}{n-m} +\frac{C_{\eta}\ln(dT)}{n-m}\E_{W_0^T,B_0^T}\left[\sum_{t=1}^T\frac{\gamma_t^2}{\eps_t^2}\norm{\gdiff_t}^2\right],
\end{align*}
where $d$ is the dimension of parameter space,  $I = [n]\backslash J$ includes indices out of $J$, $C_{\eta}:=\frac{1}{1-e^{-\eta}}$ is a constant, and $\gdiff_t:= f(W_{t-1},S_{B_t})-\g f(W_{t-1},S_{J\cap B_t})$ is the gradient difference. 
\end{ctheorem}
\begin{proof}
The proof is similar to that of Theorem~\ref{thm:fgd}. 
We still use Theorem~\ref{thm:data-pac} to prove our theorem. Let $p$ be an arbitrary real number in $(0,1/3)$. We define $P(S_J)$ as the distribution of $W'_T$ obtained by the following update rule ($W'_0:=w_0$): 
\begin{align*}
W'_t \gets W'_{t-1}+\alpha \cdot \left(W'_{t-1} - W'_{t-2}\right) - \gamma_t \g f(W'_{t-1}, S_{J_\cap B'_t}) - \eps_t\cdot \xi_t,
\end{align*}
where $B'_t$ follows the same distribution as $B_t$, and  $\xi_t$ is a discrete random variable such that for all $(a_1,..,a_d)\in \mathbb{Z}^d$:
\[\Pr[\xi_t = (a_1,...,a_d)^\top]:= \left(\sum_{i=-\infty}^{\infty}p^{i^2}\right)^{-d}\exp\left(-\sum_{k=1}^d \ln(1/p) a_k^2\right).\]
By the chain-rule of KL divergence (Lemma~\ref{lem:chain-kl}), 
we have 
\begin{align*}
\KL{W_T}{W'_T} &\leq \KL{W_0^T}{{W'}_0^T}\\
&=\sum_{t=1}^T\E_{w \sim W_0^{t-1}}\left[\KL{W_t|W_0^{t-1}=w}{W'_t|{W'}_0^{t-1}=w}\right].
\end{align*}
Again by the chain-rule of KL divergence (the sequences are $(W_0^{t-1}, B_t, W_t)$ and $({W'}_0^{t-1}, B'_t, W'_t)$), we have for any $w$:
\begin{align*}
&\KL{W_t|W_0^{t-1}=w}{W'_t|{W'}_0^{t-1}=w} = \KL{B_t|W_0^{t-1}=w}{B'_t|{W'}_0^{t-1}=w}\\
&+ \E_{B\sim B_t}\left[\KL{W_t|(W_0^{t-1},B_t)=(w,B)}{W'_t|({W'}_0^{t-1},B'_t)=(w,B)}\right].
\end{align*}
Note that $\KL{B_t|W_0^{t-1}=w}{B'_t|{W'}_0^{t-1}=w}$ is equal to zero by definition of $P(S_J)$. Moreover, conditioning on $W_0^{t-1}={W'}_0^{t-1}=w$ and $B_t = B'_t = B$, we have the KL divergence between $W_t$ and $W'_t$ is equal to $\ln(1/\Pr[\xi_t=a])$, where $a = \floor(\frac{\gamma_t}{\eps_t}(\g f(w, S_{B_t}) - \g f(w, S_{J\cap B_t})))$. Applying the method used in the proof of Theorem~\ref{alg:fgd}, we can prove 
\[\ln(1/\Pr[\xi_t=a]) \leq 3dp + \frac{\ln(1/p)\gamma_t^2}{\eps_t^2}\norm{\g f(w, S_{B}) - \g f(w, S_{J\cap B})}_2^2.\]
Thus, the KL between posterior $W_T$ and prior $W'_T$ satisfies:
\begin{align*}
\KL{W_T}{W'_T} \leq 3Tdp + \frac{\ln(1/p)}{\eps_t^2}\sum_{t=1}^T\gamma_t^2\E\norm{\g f(W_{t-1}, S_{B_t}) - \g f(W_{t-1}, S_{J \cap B_t})}_2^2.
\end{align*}
We conclude our proof by plugging it into Theorem~\ref{thm:data-pac} (setting $p=1/(Td)$). 
\end{proof}
\section{Omitted Proofs in Section~\ref{sec:ld}}\label{app:gld}

\begin{clemma}{\bf \ref{lem:gld-grad-con}}\, Let $S = (z_1,...z_n)$ be any fixed training set. $J$ is a random sequence including $m$ indices uniformly sampled from $[n]$ without replacement, and $W=(W_0,...,W_T)$ is any random sequence independent of $J$. Then the following bound holds with probability at least $1-\delta$ over the randomness of $J$:
\[\E_W\left[\sum_{t=1}^T\frac{\gamma_t^2}{\sigma_t^2}\norm{\g f(W_{t-1},S) - \g f(W_{t-1},S_J)}^2\right] \leq \frac{C_{\delta}}{m}\E_W\left[\sum_{t=1}^T\frac{\gamma_t^2}{\sigma_t^2}L(W_{t-1})^2\right],\]
where $C_{\delta} = 4+ 2\ln(1/\delta) + 5.66\sqrt{\ln(1/\delta)}$, and $L(w) = \max_{i\in[n]}\norm{\g f(w,z_i)}$.
\end{clemma}

\begin{proof}
The idea is to prove a concentration bound for the following function $\Phi$ via a modified McDiarmid inequality (Lemma~\ref{lem:new-mcd}). Define function $\Phi:[n]^m \to \R^+$ as follows:
\begin{align*}
    \Phi(J) := \sqrt{\E_W\sum_{t=1}^T\frac{\gamma_t^2}{\sigma_t^2}\norm{\g f(W_{t-1},S) - \g f(W_{t-1},S_J)}^2}.
\end{align*}
Let $J$ and $J'$ be any two ``neighboring'' sequences satisfying $J \cap J' = m - 1$. It easy to verify that $\Phi(J)$ is order-independent. 
Define $U_t = \g f(W_{t-1},S) - \g f(W_{t-1}, S_J)$ and $V_t = \g f(W_{t-1},S_J) - \g f(W_{t-1},S_{J'})$. Note that $W$ is independent of $J$. We can prove an upper bound for $\Phi(J') - \Phi(J)$.
\begin{align*}
    \Phi(J')^2 &:=\E_W\sum_{t=1}^T\frac{\gamma_t^2}{\sigma_t^2}\norm{U_t + V_t}^2\\
    &= \E_W\sum_{t=1}^T\frac{\gamma_t^2}{\sigma_t^2}(U_t^{\top}U_t + V_t^{\top}V_t) + 2\E_W\sum_{t=1}^T\frac{\gamma_t^2}{\sigma_t^2}U_t^{\top}V_t\\
    &\leq \E_W\sum_{t=1}^T\frac{\gamma_t^2}{\sigma_t^2}(U_t^{\top}U_t + V_t^{\top}V_t) + 2\sqrt{\E_W\sum_{t=1}^T\frac{\gamma_t^2}{\sigma_t^2}U_t^{\top}U_t}\sqrt{\E_W\sum_{t=1}^T\frac{\gamma_t^2}{\sigma_t^2}V_t^{\top}V_t}\\
    &=\left(\sqrt{\E_W\sum_{t=1}^T\frac{\gamma_t^2}{\sigma_t^2}\norm{U_t}^2} + \sqrt{\E_W\sum_{t=1}^T\frac{\gamma_t^2}{\sigma_t^2}\norm{V_t}^2}\right)^2\\
    &\leq \left(\Phi(J) + \frac{2}{m}\sqrt{\E_W\left[\sum_{t=1}^T\frac{\gamma_t^2}{\sigma_t^2}L(W_{t-1})^2\right]}\right)^2.
\end{align*}
The last inequality holds because $J$ and $J'$ only differ in one element. Thus $\norm{V_t}^2 \leq \frac{1}{m^2}L(W_{t-1})^2$ holds for any $W$ and $t$.
It implies $\Phi(J') \leq \Phi(J) + \frac{2}{m}\sqrt{\E_W\left[\sum_{t=1}^T\frac{\gamma_t^2}{\sigma_t^2}L(W_{t-1})^2\right]}$. Similarly, we can prove $\Phi(J) \leq \Phi(J') + \frac{2}{m}\sqrt{\E_W\left[\sum_{t=1}^T\frac{\gamma_t^2}{\sigma_t^2}L(W_{t-1})^2\right]}$. Thus we have the following holds for any $J,J'$ differing in one element:
\[|\Phi(J) - \Phi(J')| \leq \frac{2}{m}\sqrt{\E_W\left[\sum_{t=1}^T\frac{\gamma_t^2}{\sigma_t^2}L(W_{t-1})^2\right]}.\]
Applying Lemma~\ref{lem:new-mcd}, we have for any $\epsilon > 0$:
\begin{align}\label{eq:gld-mcd}
    \Pr_J\left[\Phi(J)^2 \geq (\epsilon + \E_{J}[\Phi(J)])^2\right] 
    &= \Pr_J\left[\Phi(J) - \E_{J}[\Phi(J)] \geq \epsilon \right]\notag\\
    &\leq \exp\left(\frac{-2m\epsilon^2}{4\E_W\left[\sum_{t=1}^T\frac{\gamma_t^2}{\sigma_t^2}L(W_{t-1})^2\right]}\right).
\end{align}
It remains to control the expectation:
\begin{align*}
\E_{J}[\Phi(J)] &= \E_{J}\sqrt{\E_W\left[\sum_{t=1}^T\frac{\gamma_t^2}{\sigma_t^2}\norm{\g f(W_{t-1},S) - \g f(W_{t-1},S_J)}^2\right]}\\
&\leq \sqrt{\E_W\left[\sum_{t=1}^T\frac{\gamma_t^2}{\sigma_t^2}\E_{J}\norm{\g f(W_{t-1},S) - \g f(W_{t-1},S_J)}^2\right]} \tag{$W\perp J$} 
% \\
% &\leq \sqrt{\E_W\left[\sum_{t=1}^T\frac{\gamma_t^2}{\sigma_t^2}\frac{4 L(W_{t-1})^2}{m}\right]}.
\end{align*}
For any fixed $W=(W_0,..,W_T)$ and $t\leq T$, we define $g[i]:=\g f(W_{t-1},S) - \g f(W_{t-1}, z_i)$. Let $J = (J_1,...,J_m)$. We bound the variance of $\g f(W_{t-1}, S_J)$ as follows:
\begin{align*}
\E_{J}&\norm{\g f(W_{t-1},S) - \g f(W_{t-1},S_J)}^2 = \E_{J}\left[\left(\frac{1}{m}\sum_{i=1}^m g[J_i]\right)^\top \left(\frac{1}{m}\sum_{i=1}^m g[J_i]\right)\right]\\
&= \frac{1}{m^2} \sum_{i=1}^m\sum_{j=1}^m \E_{J_i,J_j}[g[J_i]^{\top} g[J_j]]\\
&=\frac{m}{m^2}\E_{J_1}[\norm{g[J_1]}^2] + \frac{m(m-1)}{m^2}\E_{J_1,J_2}[g[J_1]^\top g[J_2]]\\
&=\frac{1}{m}\E_{J_1}[\norm{g[J_1]}^2] + \frac{m-1}{m n(n-1)}\sum_{i=1}^n\sum_{j\neq i}[g[i]^\top g[j]]\\
&=\frac{1}{m}\E_{J_1}[\norm{g[J_1]}^2] + \frac{m-1}{m n(n-1)}\left(\sum_{i=1}^n\sum_{j=1 }^n[g[i]^\top g[j]]- \sum_{i=1}^n g[i]^\top g[i]\right)\\
&\leq \frac{4L(W_{t-1})^2}{m}. \tag{$\sum_{i=1}^n g[i] = 0$ and $g[i]^\top g[i] \geq 0$}
\end{align*}
Therefore, we have
\[\E_J[\Phi(J)] \leq \sqrt{\frac{4}{m}\E_{W}\left[\sum_{t=1}^T\frac{\gamma_t^2}{\sigma_t^2}L(W_{t-1})^2\right]}.\]
Plugging the above inequality into \eqref{eq:gld-mcd} and replacing $\epsilon$ with $\sqrt{\frac{\ln(1/\delta)4\E_{W}\left[\sum_{t=1}^T \frac{\gamma_t^2}{\sigma_t^2} L(W_{t-1})^2\right]}{2m}}$, we conclude that:
\begin{align*}
\Pr_J\left[\Phi(J)^2 \leq \frac{4+ 2\ln(1/\delta) + 5.66\sqrt{\ln(1/\delta)}}{m}\E_{W}\left[\sum_{t=1}^T\frac{\gamma_t^2}{\sigma_t^2}L(W_{t-1})^2\right]\right] \geq   1-\delta.
\end{align*}
\end{proof}
\begin{ctheorem}{\bf\ref{thm:gld-bound}}\,
Suppose $J$ is a random sequence consisting for $m$ indices uniformly sampled from $[n]$ without replacement. Let $W_T$ be the output of \ref{eq:gld}. Then for any $\delta \in (0, \frac{1}{2})$ and $\eta > 0$, we have w.p. $\geq 1-2\delta$ over $S\sim \D^n$ and $J$, the following holds:
\begin{align*}
\risk(W_T,\D) \leq \eta C_{\eta} \risk(W_T,S_{[n]\backslash J}) &+  \frac{C_{\eta}\ln(1/\delta)}{n-m} + \frac{C_{\eta}C_{\delta} }{2(n-m)m}\E_W\left[\sum_{t=1}^T\frac{\gamma_t^2}{\sigma_t^2}L(W_{t-1})^2\right],
\end{align*}
where $L(w):=\max_{z \in S}\norm{f(w,z)}$, $C_{\delta} = 4+ 2\ln(1/\delta) + 5.66\sqrt{\ln(1/\delta)}$, and $C_{\eta} = \frac{1}{1-e^{-\eta}}$.
\end{ctheorem}
\begin{proof} {\bf of Theorem~\ref{thm:gld-bound}}
We use Theorem~\ref{thm:data-pac} to prove our theorem. The prior process is defined below.
\begin{align*} 
    W'_{t} \gets W'_{t-1} - \gamma_t \g f(W'_{t-1}, S_J) + \sigma_t \N(0, I_d).
\end{align*}
Then $P(S_J)$ is defined by the distribution of $W'_T$. The key is to bound the kl-divergence $\KL{W_T}{W'_T}$. Applying chain-rule of kl, we have
\begin{align*}
    \KL{W_T}{W'_T} &\leq \sum_{t=1}^T\E_{w\sim W_{t-1}}[\KL{W_t|W_{t-1}=w}{W'_t|W'_{t-1}=w}].
\end{align*}
Note that $\KL{W_t|W_{t-1}=w}{W'_t|W'_{t-1}=w}$ is equal to $\KL{\N(\mu,\sigma_t^2 I)}{\N(\mu',\sigma_t^2 I)}$, where $\mu = w - \gamma_t \g f(w,S)$ and $\mu' = w - \gamma_t \g f(w,S_J)$. One can directly compute the kl divergence of these two gaussian distributions (see e.g., \citet[Section 9]{duchi2007derivations}) to obtain
\[\KL{W_t|W_{t-1}=w}{W'_t|W'_{t-1}=w} = \frac{\norm{\mu-\mu'}^2}{2\sigma_t^2} = \frac{\gamma_t^2}{2\sigma_t^2}\norm{\g f(w,S) - \g f(w,S_J)}^2.\]
Putting this together, we have
\begin{align*}
\KL{W_T}{W'_T} \leq \E_{W}\left[\sum_{t=1}^T\frac{\gamma_t^2}{2\sigma_t^2}\norm{\g f(w,S) -\g f(w,S_J)}^2\right].
\end{align*}
Recall that $W=(W_1,..,W_T)$ is the training trajectory w.r.t. $S$. By Lemma~\ref{lem:gld-grad-con}, we can infer that w.p. at least $1-\delta$ over $J$, the above term is at most
\[\frac{C_{\delta}}{2m}\E_W\left[\sum_{t=1}^T\frac{\gamma_t^2}{\sigma_t^2}L(W_{t-1})^2\right].\]
We conclude our proof by using an union bound over $S$ and $J$.
\end{proof}

\begin{ctheorem}{\bf \ref{thm:sgld-bound}}\,
Let $W_T$ be the output of \ref{eq:sgld} when the training set is $S$, and $J$ be a random sequence with $m$ indices uniformly sampled from $[n]$ without replacement. For any $\delta \in (0, 1)$ and $m\geq 1$, we have w.p. $\geq 1-2\delta$ over $S\sim \D^n$ and $J$, the following holds:
\begin{align*}
\risk(W_T,\D) \leq \eta C_{\eta} \risk(W_T,S_I) +  \frac{C_{\eta}\ln(1/\delta)}{n-m} +
\frac{C_{\eta}}{n-m}\left(\frac{4}{b} + \frac{C_{\delta}}{2m}\right) \E_{W_0^T}\left[\sum_{t=1}^T\frac{\gamma_t^2}{\sigma_t^2}L(W_{t-1})^2\right],
\end{align*}
where $L(w):=\max_{z \in S}\norm{f(w,z)}$, $C_{\delta} = 4+ 2\ln(1/\delta) + 5.66\sqrt{\ln(1/\delta)}$, $C_{\eta}=\frac{1}{1-e^{-\eta}}$, $b$ is the batch size, and $I = [n]\backslash J$.
\end{ctheorem}
\begin{proof} Similar to Theorem~\ref{thm:gld-bound}, we  Theorem~\ref{thm:data-pac} to bound the generalization by the KL from the posterior to a data-dependent prior $P(S_J)$.
We define the prior distribution $P(S_J)$ as the output distribution of \ref{eq:sgld} trained on $S_J$. Formally, it is the distribution of $W'_T$ defined below:
\begin{align*}
  W'_{t+1} \gets W'_{t} - \gamma_t \g f(W_{t}, S_{B'_t}) + \sigma_t \N(0,I_d), 
\end{align*}
where $B'_t \sim \uni([n])^b$ is the mini-batch indices at step $t$. It is independent of other random variables including $W_{0}^{t-1}, {W'}_{0}^{t-1}$, $B_0^{t-1}$ and ${B'}_0^{t-1}$.
For any fixed $S$,
let $W$ and $W'$ be the training trajectory of posterior and prior, respectively. By the chain-rule of KL-divergence, we have
\begin{equation}\label{eq:sgld-kl-chain}
\KL{W_T}{P(S_J)} \leq \sum_{t=1}^{T}\E_{w_{t-1}\sim W_{t-1}}[\KL{W_t |W_{t-1} = w_{t-1}}{W'_t | W'_{t-1} = w_{t-1}}].
\end{equation}
For any fixed $w_{t-1}$, let $q$ and $p$ be the pdfs of $W_t|W_{t-1} = w_{t-1}$ and $W'_t|W'_{t-1} = w_{t-1}$, respectively. By the definition of \ref{eq:sgld}, we have $q = \E_{B_t}[q^{B_t}]$ and $p = \E_{B'_t}[p^{B'_t}]$, where 
\[q^{B_t} = \N(w_{t-1} - \gamma_t\g f(w_{t-1}, B_t),\sigma_t^2 I),\]
\[p^{B'_t} = \N(w_{t-1} - \gamma_t\g f(w_{t-1}, B'_t),\sigma_t^2 I).\]
By the convexity of KL-divergence, we can apply Jensen's inequality to obtain
\begin{align*}
\KL{q}{p} &= \KL{\E_{B_t,B'_t}[q^{B_t}]}{\E_{B_t,B'_t}[p^{B'_t}]}\\
&\leq \E_{B_t,B'_t}\left[\KL{q^{B_t}}{p^{B'_t}}\right]\\
&\leq \E_{B_t,B'_t}\left[\frac{\gamma_t^2\norm{\g f(w_{t-1}, B_t) - \g f(w_{t-1}, B'_t)}_2^2}{2\sigma_t^2}\right].
\end{align*}
For convenience, we define $g(A) := \frac{1}{|A|}\sum_{z\in A}\g f(w_{t-1}, z)$ for any $A \subseteq S$. Moreover, let $a$, $b$ and $c$ be $g(S_{B_t}) - g(S)$, $g(S_J) - g(S_{B'_t})$, and $g(S) - g(S_J)$, respectively. Then we can rewrite the above inequality as
\begin{align*}
\KL{q}{p} &\leq \frac{\gamma_t^2}{2\sigma_t^2}\E_{B_t,B'_t}[\norm{g(S_{B_t}) - g(S_{B'_t})}_2^2]\\
&= \frac{\gamma_t^2}{2\sigma_t^2}\E_{B_t,B'_t}[\norm{g(S_{B_t}) - g(S) + g(S_J) - g(S_{B'_t}) + g(S) - g(S_J)}_2^2]\\
&\leq \frac{\gamma_t^2}{2\sigma_t^2}\E_{B_t,B'_t}[\norm{a + b + c}_2^2]\\
&\leq \frac{\gamma_t^2}{2\sigma_t^2}\E_{B_t,B'_t}[a^{\top}a + a^{\top}(b+c) + b^{\top}b + b^{\top}(a+c) + c^{\top}c + c^{\top}(a + b)]\\
&=\frac{\gamma_t^2}{2\sigma_t^2}\E_{B_t,B'_t}[a^{\top}a + b^{\top}b + c^{\top}c].\\ 
\end{align*}
The last step is because $a = g(S_{B_t}) - g(S)$ is independent of $b=g(S_J) - g(S_{B'_t})$ and $\E[a] = \E[b] = 0$. Note that $\E[a^{\top}a] = \Var[g(S_{B_t})]$ is at most $\frac{4L(w_{t-1})^2}{b}$. Similarly, we can show that $\E[b^{\top}b] \leq \frac{4L(w_{t-1})^2}{b}$. Since $c^{\top}c$ is a constant when $S$ and $w_{t-1}$ is fixed, we have the following bound:
\begin{align*}
\KL{q}{p} \leq \frac{\gamma_t^2}{2\sigma_t^2}\left(\frac{8L(w_{t-1})^2}{b} + \norm{\g f(w_{t-1}, S) - \g f(w_{t-1}, S_J)}_2^2\right).
\end{align*}
Plugging the above inequality into \eqref{eq:sgld-kl-chain}, we have
\begin{align*}
\KL{W}{W'} \leq \sum_{t=1}^T\E_{w_{t-1}\sim W_{t-1}}&\left[\frac{4\gamma_t^2L(w_{t-1})^2}{b\sigma_t^2} + \frac{\gamma_t^2}{2\sigma_t^2}\norm{\g f(w_{t-1}, S) - \g f(w_{t-1}, S_J)}_2^2\right].
\end{align*}
Lemma~\ref{lem:gld-grad-con} shows that with probability at least $1-\delta$ over $J$ the following holds:
\begin{align*}
\E_{W_0^T}\left[\sum_{t=1}^T\frac{\gamma_t^2}{\sigma_t^2}\norm{\g f(w_{t-1}, S) - \g f(w_{t-1}, S_J)}_2^2\right] \leq \frac{C_{\delta}}{m}\E\left[\sum_{t=1}^T\frac{\gamma_t^2}{\sigma_t^2}L(W_{t-1})^2.\right]
\end{align*}
The KL divergence satisfies the following bound w.p $\geq 1-\delta$ over $J$:
\begin{align*}
\KL{W_T}{P(S_J)} \leq \E_{W_0^T}\left[\sum_{t=1}^T\frac{\gamma_t^2}{\sigma_t^2} \left(\frac{4}{b} + \frac{C_{\delta}}{2m}\right)L(W_{t-1})^2\right].
\end{align*}
We conclude our proof by plugging it into Theorem~\ref{thm:data-pac} and applying an union bound.
\end{proof}

\section{Continuous Langevin Dynamics}\label{sec:cld}
If we let the step size $\gamma_t$ approach 0, 
\ref{eq:gld} would become a coninuous diffusion process 
called Continuous Langevin Dynamics (CLD). Formally, for any fixed $S$, it is defined by
the following stochastic differential equation:
\begin{align*}\label{eq:cld}
    \rmd W_t = -\g F(W_t,S)~\rmd t+ \sqrt{2\beta^{-1}}~\rmd B_t, \quad W_0\sim \mu_0, \tag{CLD}
\end{align*}
where $F(w,S):=f(w,S) + \frac{\lambda}{2}\norm{w}^2$, $(B_t)_{t\geq 0}$ is the standard Brownian motion, and $\mu_0$ is the initial distribution. The loss function $F$ is the sum of a bounded original loss $f$ and a $\ell_2$-regularization. The main result of this section is the $O(\frac{1}{n} + \frac{1}{n^2})$ generalization bound (Theorem~\ref{thm:cld-bound}) for CLD. Before proving our main theorem, we first introduce two important mathematical tools.
\begin{lemma}[Fokker-Planck Equation] (see e.g.,~\citet{risken1996fokker} or \citet[Appendix C]{mou2018generalization})\label{lem:fokker}
For any fixed $S$, let $p(\cdot, t)$ be the pdf of $W_t$ defined in \ref{eq:cld}. 
The time evolution of $p(w,t)$ follows the Fokker-Planck equation: 
\begin{equation*}
\pp{p(w,t)}{t} = \frac{1}{\beta}\Delta p(w, t) - \g \cdot (p(w,t)\g F(\cdot, S)),
\end{equation*}
where $\Delta = \g\cdot\g$ is the Laplace operator, and $\g$ is the gradient operator w.r.t the first argument ($w$).
\end{lemma}
The following Log-Sobelev inequality for $p_t$ 
is proven in \citet[Lemma 16]{li2019generalization}, which bounds the Fisher information from below by the KL divergence.
\begin{lemma}[Log-Sobelev Inequality (LSI) for CLD]
\label{lem:lsi} Suppose $f(w,z)$ is $C$-bounded
(i.e., $|f(w,z)|\leq C$ holds for all $w,z$). Let $p_t$ be the pdf of $W_t$ in \ref{eq:cld} with $W_0 \sim \N(0, \frac{1}{\lambda \beta}I_d)$.
Then, we have for any probability density function $q$ that is absolutely continuous w.r.t. $p_t$,
the following inequality holds:
\[\KL{q}{p_t}\leq \frac{\exp(8\beta C)}{2\lambda \beta}\int_{\R^d}\norm{\g \ln\frac{q(w)}{p_t(w)}}^2 q(w)~\rmd w.\]
\end{lemma}

Applying Theorem~\ref{thm:data-pac} to CLD, we can obtain the following corollary.
The proof for bounding KL uses similar idea developed in \citet{li2019generalization}[Theorem 15].

\begin{corollary}\label{cor:cld-data} Assume the original loss function $f(w,z)$ is $C$-bounded (i.e., $|f(w,z)| \leq C$ holds for any $w$ and $z$), and the initial distribution satisfies $\rmd \mu_0 = \frac{1}{Z}e^{-\frac{\lambda\beta \norm{w}_2}{2}}~\rmd w$. Let $Q_S$ be the distribution of $W_T$ in \ref{eq:cld}. Let $J$ be a random sequence include $m$ indices uniformly sampled from $[n]$ without replacement. Then with probability at least $1-\delta$ over the randomness of $S\sim \D^n$ and $J$, the following holds:
\begin{align*}
\risk(Q_S,\D) \leq &\eta C_{\eta} \risk(Q_S,S_I) + \frac{C_{\eta}\ln(1/\delta)}{n-m}\\
&+\frac{C_{\eta}\beta}{2(n - m)}\int_{0}^T \exp\left(\frac{\lambda (t - T)}{e^{8\beta C}}\right)\E_{w\sim W_{t}}[\norm{\g F(w,S)- \g F(w,S_J)}_2^2] ~\rmd t,
\end{align*}
where $C_{\eta}:=\frac{1}{1- e^{-\eta}}$ is a constant.
\end{corollary}
\begin{proof}
Let $P(S_J):= Q_{S_J}$ be the output distribution of CLD when training data is $S_J$. From Theorem~\ref{thm:data-pac}, we can see that 
\begin{equation}\label{eq:cld-data-pac}
\risk(Q_S,\D) \leq \eta C_{\eta} \risk(Q_S,S_I) + C_{\eta}\cdot \frac{\KL{Q_S}{P(S_J)} + \ln(1/\delta)}{n-m}.
\end{equation}
The key is to control $\KL{Q_S}{P(S_J)}$. Let $W = (W_t)_{t\geq 0}$ and $W' = (W'_t)_{t\geq 0}$ be the training processes when trained on $S$ and $S_J$, respectively. Let $q_t$ and $p_t$ be the probability density function of $W_t$ and $W'_t$, respectively. Note that $\KL{Q_S}{P(S_J)}$ is equal to $\KL{q_T}{p_T}$. We first compute the upper bound of its derivative $\dd{}{t}\KL{q_t}{p_t}$ w.r.t. time $t$. 
\begin{equation}
\label{eq:dKLdtpart1}
\begin{split}
\frac{\rmd }{\rmd  t} \KL{q_t}{p_t}
&= \frac{\rmd }{\rmd  t} \int_{\R^d} q_t \log\frac{q_t}{p_t} ~\rmd  w\\
&= \int_{\R^d}\left(\frac{\rmd  q_t}{\rmd  t} \log \frac{q_t}{p_t} + q_t \cdot \frac{p_t}{q_t} \cdot \frac{\frac{\rmd q_t}{\rmd  t}p_t - q_t \frac{\rmd p_t}{\rmd  t}}{p_t^2}\right) ~\rmd  w\\
&= \int_{\R^d}\left(\frac{\rmd q_t}{\rmd  t} \log\frac{q_t}{p_t}\right) ~\rmd w- \int_{\R^d}\left(\frac{q_t}{p_t}\frac{\rmd p_t}{\rmd  t}\right)~\rmd w
\end{split}
\end{equation}
According to Fokker-Planck Equation (Lemma~\ref{lem:fokker}), we can compute the derivative of CLD pdfs w.r.t time $t$:
\begin{gather}
\frac{\p q_t}{\p t} = \frac{1}{\beta} \Delta q_t + \g \cdot (q_t\g F(\cdot,S))\notag,\quad\quad
\frac{\p p_t}{\p t} = \frac{1}{\beta} \Delta p_t + \g \cdot (p_t\g F(\cdot, S_J))\notag.
\end{gather}
It follows that
\begin{align*}
I &:= \int_{\R^d}\left(\frac{\rmd q_t}{\rmd  t} \log\frac{q_t}{p_t}\right) ~\rmd w\\
&= \int_{\R^d} \left(\frac{1}{\beta} \Delta q_t + \g \cdot (q_t\g F(w,S))\right)\log\frac{q_t}{p_t} ~\rmd w \\
&= \frac{-1}{\beta}\int_{\R^d}\langle \g \log \frac{q_t}{p_t}, \g q_t\rangle ~\rmd w - \int_{\R^d}\langle \g \log \frac{q_t}{p_t}, q_t\g F(w,S)\rangle ~\rmd w \tag{integration by parts},
\end{align*}
and 
\begin{align*}
J &:= \int_{\R^d}\left(\frac{q_t}{p_t} \frac{\rmd  p_t}{\rmd  t}\right) ~\rmd w\\
&= \int_{\R^d}\frac{q_t}{p_t}\left(\frac{1}{\beta}\Delta p_t + \g \cdot (p_t\g F(w,{S_J}))\right) ~\rmd w\\
&= \frac{-1}{\beta} \int_{\R^d}\langle \g \frac{q_t}{p_t}, \g p_t\rangle ~\rmd w - \int_{\R^d} \langle \g \frac{q_t}{p_t}, p_t \g F(w,{S_J})\rangle ~\rmd w \tag{integration by parts}.
\end{align*}
Together with \eqref{eq:dKLdtpart1}, we have 
\begin{align*}
\frac{\rmd}{\rmd t} \KL{q_t}{p_t}
&= I - J\\
&= \frac{-1}{\beta}\int_{\R^d}\left(\langle\frac{\g q_t}{q_t} - \frac{\g p_t}{p_t},\g q_t\rangle - \langle\frac{\g q_t}{p_t} - \frac{q_t \g p_t}{p_t^2},\g p_t\rangle\right) ~\rmd w\\
&-\int_{\R^d}\left(\langle\g \log\frac{q_t}{p_t},q_t \g F(w,{S})\rangle - \frac{q_t}{p_t}\langle \g \log \frac{q_t}{p_t}, p_t \g F(w,{S_J})\rangle\right) ~\rmd w\\
&= \frac{-1}{\beta}\int_{\R^d} q_t \left\|\g \log \frac{q_t}{p_t}\right\|_2^2 ~\rmd w + \int_{\R^d}q_t\langle \g \log \frac{q_t}{p_t},\g F(w,{S}) - \g F(w,{S_J})\rangle ~\rmd w\\
&\leq \frac{-1}{2\beta}\int_{\R^d}q_t\left\|\g \log \frac{q_t}{p_t}\right\|_2^2 ~\rmd w + \frac{\beta}{2}\int_{\R^d} q_t \left\|\g F(w,{S}) - \g F(w,{S_J})\right\|_2^2 ~\rmd w.
\end{align*}
The last step holds because $\langle \mathbf{a}/\sqrt{\beta}, \mathbf{b}\sqrt{\beta}\rangle \leq \frac{\left\|\mathbf{a}\right\|_2^2}{2\beta} + \frac{\beta\left\|\mathbf{b}\right\|_2^2}{2}$.
By the Log-Sobolev inequality for CLD (Lemma~\ref{lem:lsi}), we have
\[\int_{\R^d}q_t\left\|\g \log \frac{q_t}{p_t}\right\|_2^2 ~\rmd w \geq \frac{2\lambda \beta}{\exp(8\beta C)}\KL{q_t}{p_t}.\]
Hence the derivative satisfies the following bound:
\[\dd{}{t}\KL{q_t}{p_t} \leq \frac{-\lambda}{\exp(8\beta C)}\KL{q_t}{\pi_t} +  \frac{\beta}{2}\E_{W_t}[\norm{\g F(W_t,S)- \g F(W_t, S_J)}_2^2].\]
Let $\alpha = \frac{\lambda}{e^{8\beta C}}$, $y(t):= \KL{q_t}{p_t}$, and $g(t) = \frac{\beta}{2}\E_{W_t}[\norm{\g F(W_t,S)- \g F(W_t, S_J)}_2^2]$. Then we can rewrite the above inequality as 
\[y(t)' \leq -\alpha y(t) + g(t), \qquad y(0) = 0.\]
Solving this inequality, we have
\begin{align*}
\KL{q_T}{p_T} &\leq \frac{\beta}{2}\int_{0}^T \exp\left(\frac{\lambda (t - T)}{e^{8\beta C}}\right)\E_{W_{t}}[\norm{\g F(W_t,S)- \g F(W_t, S_J)}_2^2] ~\rmd t.
\end{align*}
\end{proof}
The following Lemma~\ref{lem:cld-grad-con} demonstrates that the integral of the gradient difference $\norm{\g F_S - \g F_{S_J}}_2^2$ enjoys a concentration property like Lemma~\ref{lem:gld-grad-con}.
\begin{definition}[Lipschitz] A differentiable function is $L$-Lipschitz if and only if $\norm{\g_w f(w,z)} \leq L$ holds for any $w\in \R^d$.
\end{definition}
\begin{lemma}\label{lem:cld-grad-con}
Suppose the loss function $f$ is $L$-Lipschitz. Let $S\in \Z^n$ be any fixed training set, and $W=(W_t)_{t \in [0, T]}$ be any random process. For any $\alpha>0$, we have the following bound holds w.p. $\geq 1-\delta$ over the randomness of $J$ ($m$ indices sampled from $[n]$ without replacement):
\[\E_{W}\left[\int_{0}^T e^{\alpha(t-T)}\norm{f(W_t, S) - f(W_t, S_J)}_2^2~\rmd t\right] \leq \frac{C_{\delta}L^2(1-e^{-\alpha T})}{\alpha m},\]
where $C_{\delta} = 4+ 2\ln(1/\delta) + 5.66\sqrt{\ln(1/\delta)}$.
\end{lemma}
\begin{proof}
Define function $\Phi :[n]^m \to \R^+$ as follows:
\[\Phi(J):= \sqrt{\E_{W}\left[\int_{0}^T e^{\alpha(t-T)}\norm{f(W_t, S) - f(W_t, S_J)}_2^2~\rmd t\right]}.\]
Let $J$ and $J'$ be any two ``neighboring'' index-sets. In other words, they should satisfy $J \cap J' = m-1$. Similar to the proof of Lemma~\ref{lem:gld-grad-con}, we first show that $|\Phi(J)-\Phi(J')|$ is small. Formally, define $U_t = \g f(W_{t-1}, S) - \g f(W_{t-1}, S_J)$ and $V_t = \g f(W_{t-1},S_J) - \g f(W_{t-1}, S_{J'})$. We have
\begin{align*}
\Phi(J')^2 &:= \E_{W}\left[\int_{0}^T e^{\alpha(t-T)}\norm{U_t + V_t}_2^2~\rmd t\right]\\
&= \E_{W}\left[\int_{0}^T e^{\alpha(t-T)}\left(U_t^\top U_t + V_t^\top V_t \right)~\rmd t\right] + 2\E_{W}\left[\int_{0}^T e^{\alpha(t-T)}U_t^\top V_t ~\rmd t\right]\\
&\leq \E_{W}\left[\int_{0}^T e^{\alpha(t-T)}\left(\norm{U_t}_2^2 + \norm{V_t}_2^2 \right)~\rmd t\right]\\ 
&\qquad+ 2\sqrt{\E_{W}\left[\int_{0}^T e^{\alpha(t-T)}\norm{U_t}^2_2 ~\rmd t\right]}\sqrt{\E_{W}\left[\int_{0}^T e^{\alpha(t-T)}\norm{V_t}^2_2 ~\rmd t\right]}\\
&= \left(\sqrt{\E_{W}\left[\int_{0}^T e^{\alpha(t-T)}\norm{U_t}^2_2 ~\rmd t\right]} + \sqrt{\E_{W}\left[\int_{0}^T e^{\alpha(t-T)}\norm{V_t}^2_2 ~\rmd t\right]}\right)^2\\
&= \left(\Phi(J) + \sqrt{\E_{W}\left[\int_{0}^T e^{\alpha(t-T)}\norm{V_t}^2_2 ~\rmd t\right]}\right)^2.
\end{align*}
For any fixed $W$, we have 
\begin{align*}
\int_{0}^T e^{\alpha(t-T)}\norm{U_t}_2^2~\rmd t &\leq \int_{0}^T e^{\alpha(t-T)}\frac{4L^2}{m^2}~\rmd t\\
&= \frac{4L^2(1 - e^{\alpha T})}{\alpha m^2}.
\end{align*}
Plugging it into the above inequality, we obtain
\[\Phi(J')^2 \leq \left(\Phi(J) + \frac{2L}{m}\sqrt{\frac{1 - e^{-\alpha T}}{\alpha}}\right)^2.\]
The other direction can be proved in a same way. Therefore, for any $J$ and $J'$ that are different in only one element, we have:
\[|\Phi(J) - \Phi(J')| \leq \frac{2L}{m}\sqrt{\frac{1 - e^{\alpha T}}{\alpha}}.\]
Applying Lemma~\ref{lem:new-mcd}, one can infer that for any $\epsilon > 0$:
\[\Pr_J\left[ \Phi(J) - \E_{J}[\Phi(J)]\geq \epsilon\right] \leq \exp\left(\frac{-2m\epsilon^2}{4L^2(1-e^{-\alpha T})/\alpha}\right).\]
It further implies that
\begin{equation}\label{eq:cld-mcd}
\Pr_J\left[ \Phi(J)^2 \geq (\epsilon + \E_{J}[\Phi(J)])^2\right] \leq \exp\left(\frac{-2m\epsilon^2}{4L^2(1-e^{-\alpha T})/\alpha}\right).    
\end{equation}
It remains to bound the expectation:
\begin{align*}
\E_J[\Phi(J)] &= \E_J\sqrt{\E_{W}\left[\int_{0}^T e^{\alpha(t-T)}\norm{f(W_t, S) - f(W_t, S_J)}_2^2~\rmd t\right]}\\
&\leq \sqrt{\E_{W}\left[\int_{0}^T e^{\alpha(t-T)}\E_J[\norm{f(W_t, S) - f(W_t, S_J)}_2^2]~\rmd t\right]}\\
&\leq \sqrt{\int_{0}^T e^{\alpha(t-T)}\frac{4L^2}{m}~\rmd t} \tag{a}\\
&= \frac{2L}{\sqrt{m}}\sqrt{\frac{1-e^{-\alpha T}}{\alpha}},
\end{align*}
Plugging it into \eqref{eq:cld-mcd} and replacing $\epsilon$ with $\sqrt{\frac{4L^2(1-e^{-\alpha T})/\alpha\cdot \ln(1/\delta)}{2m}}$, we can conclude the proof. It remains to prove (a) in the above inequality. For any fixed $W$ and $t\in [0,T]$, we define $g[i]:=\g f(W_t,S) - \g f(W_t, z_i)$. Let $J = (J_1,...,J_m)$. We bound the variance of $\g f(W_t, S_J)$ as follows:
\begin{align*}
\E_{J}&\norm{\g f(W_t,S) - \g f(W_t,S_J)}^2 = \E_{J}\left[\left(\frac{1}{m}\sum_{i=1}^m g[J_i]\right)^\top \left(\frac{1}{m}\sum_{i=1}^m g[J_i]\right)\right]\\
&= \frac{1}{m^2} \sum_{i=1}^m\sum_{j=1}^m \E_{J_i,J_j}[g[J_i]^{\top} g[J_j]]\\
&=\frac{m}{m^2}\E_{J_1}[\norm{g[J_1]}^2] + \frac{m(m-1)}{m^2}\E_{J_1,J_2}[g[J_1]^\top g[J_2]]\\
&=\frac{1}{m}\E_{J_1}[\norm{g[J_1]}^2] + \frac{m-1}{m n(n-1)}\sum_{i=1}^n\sum_{j\neq i}[g[i]^\top g[j]]\\
&=\frac{1}{m}\E_{J_1}[\norm{g[J_1]}^2] + \frac{m-1}{m n(n-1)}\left(\sum_{i=1}^n\sum_{j=1 }^n[g[i]^\top g[j]]- \sum_{i=1}^n g[i]^\top g[i]\right)\\
&\leq \frac{4L^2}{m}. \tag{$\sum_{i=1}^n g[i] = 0$ and $g[i]^\top g[i] \geq 0$}
\end{align*}
\end{proof}
Now we are ready to prove our generalization bound for CLD.
\begin{theorem}\label{thm:cld-bound}
Assume the original loss function $f(w,z)$ is $C$-bounded (i.e. $|f(w,z)| \leq C$ holds for all $w,z$), and $W_0 \sim \N(0,\frac{1}{\lambda \beta} I_d)$. Let $W_T$ be the output of \ref{eq:cld}. Then, for any $\delta\in (0,1)$ and $\eta > 0$, we have the following inequality holds with probability at least $1-2\delta$ over the randomness of $S\sim \D^n$ and $J$ ($m$ indices uniformly sampled from $[n]$ without replacement):
\begin{align*}
\risk(W_T,\D) \leq & \eta C_{\eta} \risk(W_T,S_I) + \frac{C_{\eta}\ln(1/\delta)}{n-m}+\frac{C_{\eta}C_{\delta}\beta L^2\cdot e^{8\beta C}\left(1-\exp({-\frac{\lambda T}{e^{8\beta C}}})\right)}{2\lambda (n - m)m},
\end{align*}
where $C_{\delta} = 4 + 2 \ln (1/\delta) + 5.66\sqrt{\ln (1/\delta)}$, $C_\eta = \frac{1}{1- e^{-\eta}}$, and $I=[n]\backslash J$.
\end{theorem}
\begin{proof}
Let $W = (W_t)_{t \in [0,T]}$ be the training trajectory of CLD when dataset is $S$. Applying Lemma~\ref{lem:cld-grad-con} with $\alpha = \frac{\lambda }{e^{8\beta C}}$, we have w.p. $\geq 1-\delta$ over $J$:
\begin{align*}
&\int_{0}^T \exp\left(\frac{\lambda(t-T)}{e^{8\beta C}}\right) \E_{W_t}[\norm{\g F(W_t,S) - \g F(W_t,S_J)}_2^2]~\rmd t\\
&= \E_W\left[\int_{0}^T \exp\left(\frac{\lambda(t-T)}{e^{8\beta C}}\right) \norm{\g f(W_t, S) - \g f(W_t,S_J)}_2^2~\rmd t\right]\\
&\leq \frac{C_{\delta}L^2\cdot e^{8\beta C}\left(1-\exp({-\frac{\lambda T}{e^{8\beta C}}})\right)}{\lambda m}.
\end{align*}
we conclude our proof by plugging it into Corollary~\ref{cor:cld-data} and use an union bound over $S$ and $J$.
\end{proof}
\section{Experimental Details}\label{app:exp-detail}

We train our model on a single server equipped with Intel Xeon CPU (2.40GHZ, 16 cores), 256G memory, and GeForce GTX 1080 Ti (11G) GPU.
\paragraph{Models.} For MNIST experiments, we use a CNN defined as follows (conv kernel size is $5\times 5$):
\begin{table}[H]
\centering
\begin{tabular}{c| c |c |c}
     \hline
     1 & 2 & 3 & 4 \\\hline
     conv(32) + relu & conv(512) & fc(1024) + relu & fc(10)\\
     conv(32) + relu & relu & &\\
     maxpool(2) & & &\\
     \hline
\end{tabular}
\end{table}
For CIFAR10 experiments, we use a modifed version (turnoff the BatchNorm and Dropout) of SimpleNet~\citep{hasanpour2016lets} which is defined below (convolution kernel size is $3 \times 3$):
\begin{table}[H]
\begin{tabular}{c |c |c |c |c |c}
 \hline
 1 & 2 & 3 & 4 & 5 & 6\\\hline
 conv(64) + relu & conv(128) + relu & conv(256) + relu & conv(512) + relu & conv(2048) + relu & conv(256)\\
 conv(64) + relu & conv(128) + relu & conv(256) + relu & & conv(256) + relu &\\
 conv(64) + relu & conv(256) + relu & & & &\\
 conv(64) + relu &  & & & &\\
 maxpool(2) & maxpool(2) & maxpool(2) & maxpool(2) & maxpool(2) & fc(10)\\
 \hline
\end{tabular}
\end{table}
\paragraph{Train FGD on MNIST.} We train a CNN defined above by FGD ($\eps_t=0.005$, momentum $\alpha=0.9$, $m=n/2=30000$) 20 times (with different initialization $w_0$ and $J$). We plot the means and stds (error bar) in Figure~\ref{fig:fgd-mnist-exp-a} and \ref{fig:fgd-mnist-exp-b}. Recall that the bound at step $T$ is the RHS of Theorem~\ref{thm:fgd} ($\eta = 1$, $\delta=0.1, d = 1,407,370$):
\begin{equation*}
\bd = \frac{1}{1-e^{-1}}\left[\risk(W_T,S_I) + \frac{\ln(10) + 3}{30000} + \frac{\ln(d T)}{30000}\sum_{t=1}^T\frac{\gamma_t^2}{\eps_t^2}\norm{\gdiff_t}^2\right],
\end{equation*}
where $\gdiff_t:=\g f(W_{t-1},S) - \g f(W_{t-1}, S_J)$ is deterministic when $w_0, J$ are fixed. 

We also study how the prior size $m$ affects the squared norm of gradient difference $\norm{\gdiff_t}^2$. We test 9 different choices of $m$ (from 1000 to 9000). For each prior size $m=|J|$, we run our experiment 30 times and report the means and stds (error bar) in Figure~\ref{fig:fgd-mnist-exp-c}.

\paragraph{Train FSGD on CIFAR10.} It should be very time-consuming to train our SimpleNet on CIFAR10 by FGD as it requires computing full gradient and demands for more training steps. Hence we use the stochastic FSGD (Algorithm~\ref{alg:fsgd}) to train our model. The learning rate $\gamma_t$ and the precision $\eps_t$ are set to $0.001\cdot 0.9^{\lfloor \frac{t}{200} \rfloor}$ and $0.004$, respectively. At each step, the random mini-batch with size $b=2000$ is made up of $1000$ 
indices uniformly sampled from $I$ and $1000$ indices uniformly sampled from $J$. We run FSGD ($m=n/5 = 10000$) 15 times and report the means and stds (error bar) in Figure~\ref{fig:fsgd-cifar10-exp-a}, \ref{fig:fsgd-cifar10-exp-c} and~ \ref{fig:fsgd-cifar10-exp-d}, where the bound is the RHS of Theorem~\ref{thm:fsgd} ( $\eta=2, \delta=0.1$, $d=18,072,202$):
\begin{equation*}
\bd = \frac{1}{1-e^{-3}}\left[3\risk(W_T,S_I) + \frac{\ln(10) + 3}{40000} + \sum_{t=1}^T\frac{\gamma_t^2}{\eps_t^2}\norm{\gdiff_t}^2\right],
\end{equation*}
where $\gdiff_t:=\g f(W_{t-1}, S_{B_t}) - \g f(W_{t-1}, S_{J \cap B_t})$ is the gradient difference w.r.t. this run. Note that in Theorem~\ref{thm:fsgd}, the bound should take expectation over the randomness of $B_0^T$. However, one can view the random seed generating $B_0^T$ is fixed so that FSGD becomes a deterministic algorithm.

\paragraph{Random labels.}
We conduct the random label experiment designed in \citet{ZhangBHRV17}.
We replace the true labels of some training samples with random labels.
The portion of random labels is specified by $p$ ($0\leq p\leq 1$).
Concretely, if the training dataset includes $n$ samples, the labels of $np$ samples (randomly chosen) are replaced with random labels. 
We use the same neural network architectures as above. 
In Figure~\ref{fig:fgd-mnist-random}, we train a CNN defined above by FGD ($\eps_t=0.0005$, $\gamma_t=0.0005\times0.9^{\lfloor\frac{t}{150}\rfloor}$ momentum $\alpha=0.9$, $m=n/2=30000$) 20 times per random portion $p$. 
In Figure~\ref{fig:fsgd-cifar-random}, we train a SimpleNet by FSGD ($\gamma_t=0.001\cdot 0.9^{\lfloor \frac{t}{200} \rfloor}$ and $\epsilon_t=0.004$) 10 times per random portion $p$.
One can see that even for such datasets with larger true test errors,
our bounds are still non-vacuous.

\begin{figure}[t]
    \centering
    \subcaptionbox{training accuracy\label{fig:fgd-mnist-random-a}}
      {\includegraphics[width=0.3\linewidth]{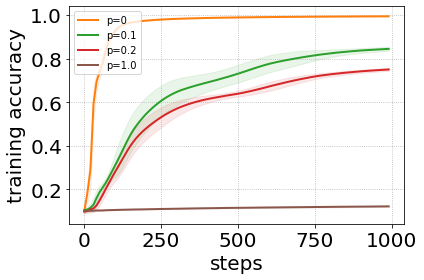}}
    \subcaptionbox{test error\label{fig:fgd-mnist-random-b}}
      {\includegraphics[width=0.3\linewidth]{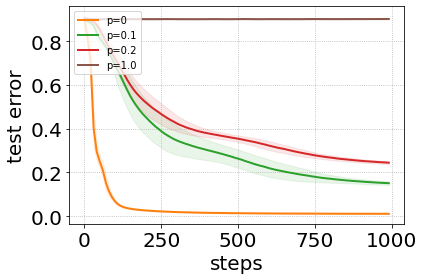}}
    \subcaptionbox{our bound\label{fig:fgd-mnist-random-c}}
      {\includegraphics[width=0.3\linewidth]{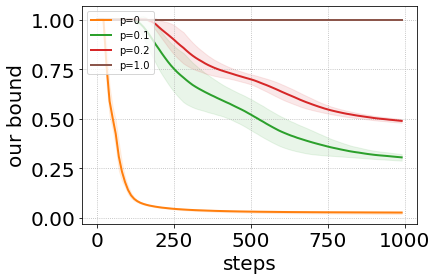}}
    \caption{\label{fig:fgd-mnist-random}Random labels (MNIST + FGD). Here $p$ is the portion of random labels.}
\end{figure}
\begin{figure}[t]
    \centering
    \subcaptionbox{\label{fig:fsgd-cifar10-random-a}}
      {\includegraphics[width=0.3\linewidth]{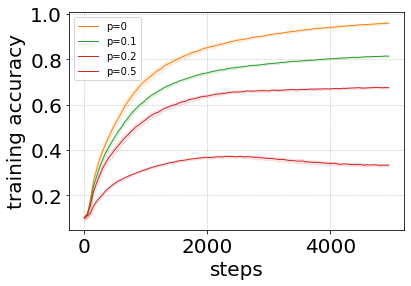}}
    \subcaptionbox{\label{fig:fsgd-cifar10-random-b}}
      {\includegraphics[width=0.3\linewidth]{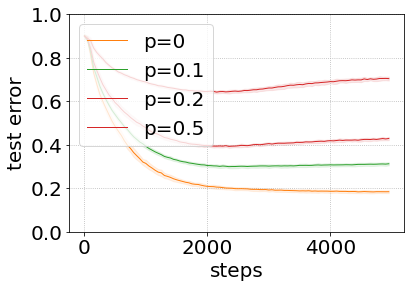}}
    \subcaptionbox{\label{fig:fsgd-cifar10-random-c}}
      {\includegraphics[width=0.3\linewidth]{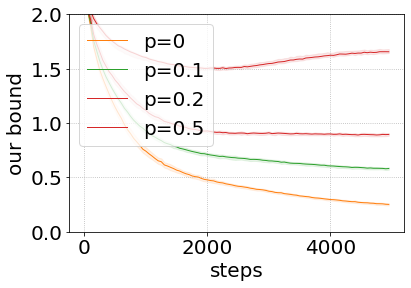}}
    \caption{\label{fig:fsgd-cifar-random} Random labels (FSGD + CIFAR10).}
\end{figure}

\paragraph{FGD vs GD} We attempt to show that the performance of FGD (Algorithm~\ref{alg:fgd}) with reasonable precision $\eps_t$ is similar to the traditional Gradient Descent (GD) defined below (with momentum $\alpha$):
\begin{align*}
    W_t \gets W_{t-1} +\alpha(W_{t-1}-W_{t-2}) + \gamma_t \g f(W_{t-1}, S). \tag{GD}
\end{align*}
We train a CNN defined above on MNIST by GD ($\gamma_t = 0.005\times 0.9^{\lfloor\frac{t}{150}\rfloor}$, $\alpha=0.9$). And we compare the training curves with that of FGD under the same hyper-parameter setting ($\eps_t=\gamma_t=0.005$, $\alpha=0.9$). We repeat our experiment on GD $25$ times. The result is shown in Figure~\ref{fig:fgd-gd-mnist}.
\begin{figure}
    \centering
    \subcaptionbox{FGD}
      {\includegraphics[width=0.3\linewidth]{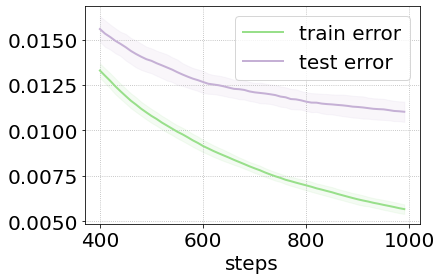}}
    \subcaptionbox{GD}
      {\includegraphics[width=0.3\linewidth]{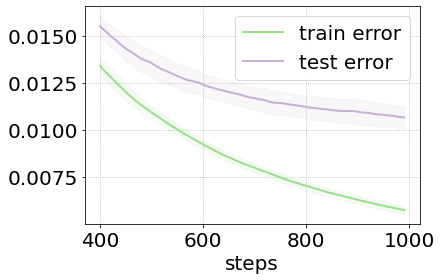}}
    \subcaptionbox{together}
      {\includegraphics[width=0.3\linewidth]{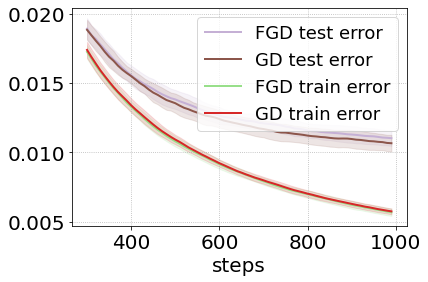}}
    \caption{\label{fig:fgd-gd-mnist}MNIST: FGD vs GD. In (c), we plot FGD and GD together. As we can see that, the curves of FGD and GD are almost coincident.}
\end{figure}
\begin{figure}
    \centering
    \subcaptionbox{FSGD}
      {\includegraphics[width=0.3\linewidth]{figs/cifar10/cifar_0.png}}
    \subcaptionbox{SGD}
      {\includegraphics[width=0.3\linewidth]{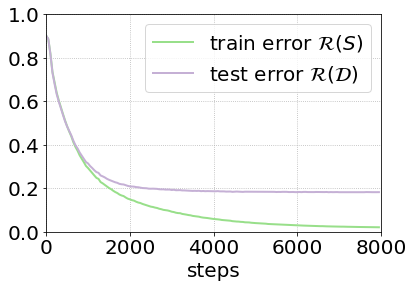}}
    \subcaptionbox{together}
      {\includegraphics[width=0.3\linewidth]{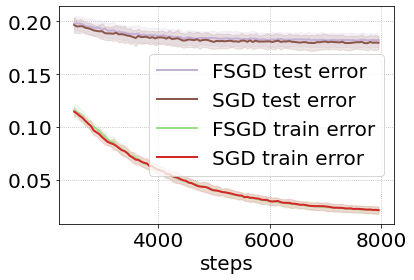}}
    \caption{CIFAR10: FSGD vs SGD. \label{fig:fsgd-sgd-cifar}}
\end{figure}
As we can see from the figures, the optimization as well as generalization performance of GD and FGD are close.
\paragraph{FSGD vs SGD} We also show that the performance of FSGD (Algorithm~\ref{alg:fsgd}) with reasonable precision $\eps_t$ is very close to the ordinary Stochastic Gradient Descent (SGD) defined below (with momentum $\alpha$ and a). The only difference is that we sample a mini-batch $B_t$ before each step.:
\begin{align*}
    W_t \gets W_{t-1} +\alpha(W_{t-1}-W_{t-2}) + \gamma_t \g f(W_{t-1}, S_{B_t}). \tag{SGD}
\end{align*}
We train a SimpleNet defined above on CIFAR10 by SGD ($\gamma_t = 0.001$, $\alpha=0.99$). And we compare the training curves with that of FSGD under the same hyper-parameter setting ($\eps_t=0.001,\gamma_t=0.001\times 0.9^{\lfloor\frac{t}{200}\rfloor}, \alpha=0.99$). We repeat our experiment on SGD $10$ times. The result is shown in Figure~\ref{fig:fsgd-sgd-cifar}.
As we can see from the figures, the optimization as well as generalization performance of FSGD are close to SGD.

\end{document}